\documentclass[]{fairmeta}
\usepackage{xspace}

\makeatletter
\DeclareRobustCommand\onedot{\futurelet\@let@token\@onedot}
\def\@onedot{\ifx\@let@token.\else.\null\fi\xspace}

\makeatother

\usepackage{wrapfig}
\usepackage{colortbl}
\usepackage{color, xcolor} 
\usepackage{listings}
\usepackage{enumitem}

\newcommand{\method}{\textsc{NAVA-WAM}}
\usepackage{tabularx}
\usepackage{array}
\usepackage{wrapfig}
\usepackage[table]{xcolor}
\usepackage{amsthm}
\newtheorem{proposition}{Proposition}[section]
\usepackage{algorithm}
\usepackage{algorithmic}

\definecolor{row}{RGB}{235, 245, 251}

\definecolor{cvprblue}{rgb}{0.21,0.49,0.74}
\usepackage{makecell}
\usepackage{wrapfig}
\usepackage{amssymb}
\usepackage{graphicx}   
\usepackage{subcaption} 
\usepackage{booktabs}   
\usepackage{multirow}   
\newcolumntype{C}[1]{>{\centering\arraybackslash}m{#1}}
\newcommand{\projectpage}{\href{https://zhaochongan.github.io/projects/NAVA-WAM}{\textcolor{magenta}{Project Page}}\xspace}

\title{Native Action-Prior Learning from Videos for World Action Models}

\author[1,2,\dagger]{Zhaochong An}
\author[1]{Fei Zhang}
\author[1]{Menglin Jia}
\author[1]{Duncan Frost}
\author[1]{Zijian Zhou}
\author[1]{Yikai Wang}
\author[4]{Xudong Wang}
\author[1]{Aditya Patel}
\author[1]{Belinda Zeng}
\author[1]{Tao Xiang}
\author[2]{Serge Belongie}
\author[3]{Amir Bar}
\author[1,*]{Sen He}

\affiliation[1]{Meta AI}
\affiliation[2]{University of Copenhagen}
\affiliation[3]{Imperial College London}
\affiliation[4]{University of California, Berkeley}

\contribution[\dagger]{Work done at Meta}
\contribution[*]{Project lead and corresponding author}

\abstract{
World action models integrate future visual dynamics with robot action prediction, but their scalability remains limited by the need for action-annotated robot trajectories. 
Observation-only videos contain rich evidence about interaction dynamics, but existing approaches typically use them either to pretrain visual representations that must later be adapted for control, or to infer latent actions that are subsequently grounded to robot commands.
We present \method{}, which introduces \emph{native action-prior learning} by directly pretraining the action policy from observation-only videos, avoiding indirect representation-to-control transfer or a separate latent-action model.
Our training consists of two stages.
First, we pretrain on observation-only videos, where future-video flow-matching supervision over visual transitions is propagated through transition-structured joint attention to optimize the Action-DiT and learn action-relevant priors.
Second, we use action-labeled demonstrations to post-train the Action-DiT for robot control through joint video--action flow matching, while asymmetric attention decouples the visual branch from iterative action denoising and enables efficient action-only inference.
Extensive experiments show that \method{} consistently outperforms prior approaches under both in-distribution and out-of-distribution settings, while demonstrating strong action-label efficiency and effective real-robot generalization.
These results establish native action-prior learning as an effective approach to directly pretrain action policies from observation-only videos, providing a scalable path beyond action-labeled robot data.
}

\date{October 5, 2026}
\metadata[Project page]{\url{https://zhaochongan.github.io/projects/NAVA-WAM/}}

\begin{document}

\maketitle

\section{Introduction}
\label{sec:intro}

World action models (WAMs)~\citep{zhu2025uwm,li2025uva,ye2026dreamzero,li2026lingbotva,bi2025motus} incorporate future visual dynamics into action learning, with strong potential to improve policy generalization.
However, scaling WAMs remains challenging because they rely on action-annotated robot trajectories that are costly to collect and hard to standardize across embodiments.
In contrast, large-scale observation-only videos provide abundant information about physical interactions and state transitions without action labels.
The key challenge is therefore how to leverage such videos to learn priors that effectively improve action policy learning.

Existing approaches largely follow two paradigms for exploiting observation-only videos.
\textit{Vision representation learning} first learns visual or predictive representations from videos and then trains an action policy on top of them using action-labeled robot demonstrations~\citep{nair2023r3m,hu2025vpp,assran2025vjepa2,sun2026vla}.
Because the learned video prior remains encoded in a vision-centric representation, its connection to executable actions is established only during downstream robot training, introducing an indirect representation-to-control transfer. 
\textit{Latent-action learning} instead infers action-like variables from visual transitions and uses them as intermediate supervision for subsequent policy learning~\citep{schmidt2024lapo,ye2025lapa,chen2025moto,bu2025univla,bi2025motus,chen2026lawam}.
While more action-oriented, this paradigm requires an additional latent-action modeling stage and relies on complex architectural or optimization constraints~\citep{yang2026learning,wang2026repwam,zhang2026lingbotva2} to avoid shortcut solutions in which the latent action simply copies information from the target observation. 
These constraints increase modeling complexity and can restrict the expressiveness of the latent-action space, thereby limiting the supervision provided to the policy and impairing its control learning.

These limitations motivate a simple question: 
\emph{Can observation-only videos directly optimize an action policy without relying on an intermediate interface?} 
We answer this question with \method{}, which introduces \emph{native action-prior learning} for world action models.
As illustrated in Fig.~\ref{fig:pretrain_data}, instead of first learning a separate visual representation or latent-action space and then using it in policy training, \method{} uses future-video prediction to directly pretrain the action policy. 
This creates a direct learning path from scalable observation-only videos to the action policy, avoiding intermediate representation constraints and enabling more effective policy learning.

\begin{figure}[t]
    \centering
    \includegraphics[width=\linewidth]{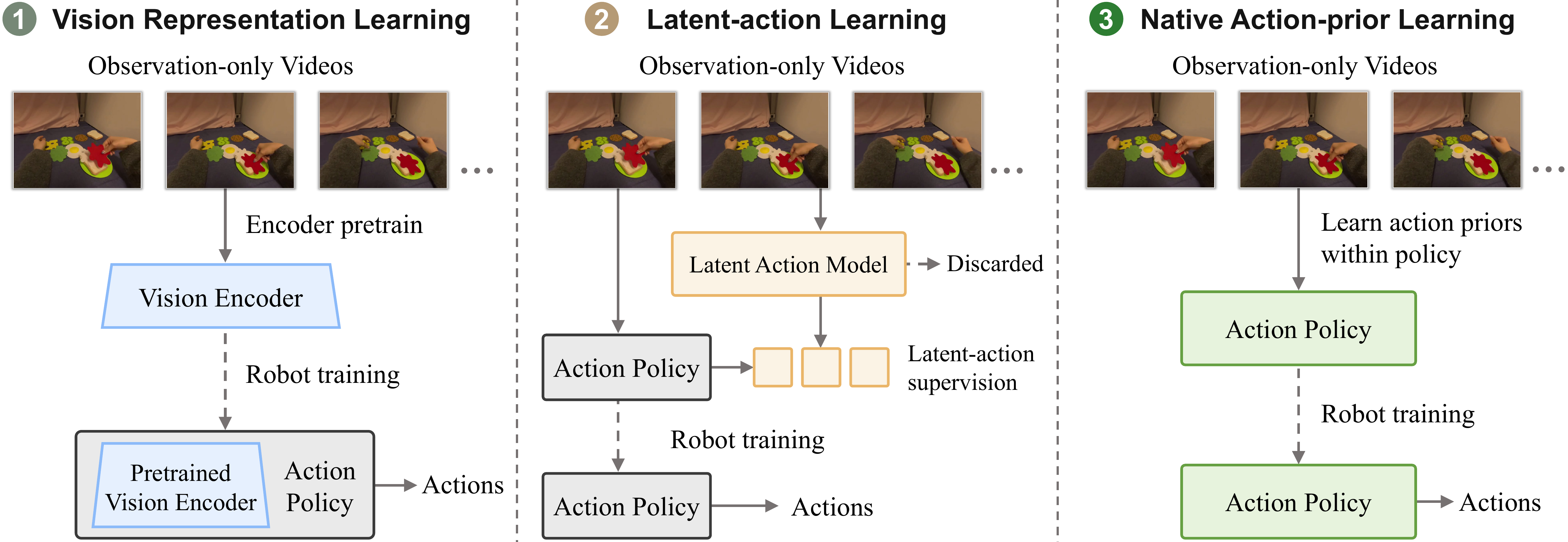}
    \caption{
    \textbf{Paradigms of leveraging observation-only videos for robot policy learning.}
    \textit{(1) Vision representation learning} first learns visual representations from videos within a vision encoder and then trains an action policy on top, introducing an indirect representation-to-control transfer.
    \textit{(2) Latent-action learning} extracts latent actions from visual transitions as intermediate supervision for policy learning, introducing an additional modeling stage and a potential representational bottleneck.
    \textit{(3) Native action-prior learning} directly uses video dynamics to optimize the action policy, avoiding additional intermediate interfaces and representational bottlenecks.
    }
    \label{fig:pretrain_data}
\end{figure}

Specifically, \method{} adopts a joint Mixture-of-Transformers (MoT) architecture~\citep{liang2024mixture}, comprising a Video-DiT and an Action-DiT that interact through joint attention while retaining modality-specific parameters.
During action-free pre-training, the Video-DiT is frozen and the Action-DiT is optimized solely through future-video flow-matching supervision.
Our transition-structured joint attention allows the Action-DiT to model visual transitions while participating in future-state prediction, enabling the video objective to directly optimize the action policy.
In this way, the Action-DiT learns action-relevant priors directly from visual transitions without requiring an intermediate interface.
During downstream post-training, we jointly optimize the video and action streams on action-labeled robot demonstrations through video--action flow matching, adapting the pretrained Action-DiT to embodiment-specific control.
We make cross-stream attention asymmetric to decouple the visual stream from the evolving action sample, allowing its representations to be cached and enabling efficient action-only inference without explicitly generating future videos.

We evaluate \method{} on the simulation benchmarks LIBERO/LIBERO-Plus~\citep{liu2023libero,fei25libero-plus} and RoboTwin~2.0~\citep{chen2025robotwin} under both in-distribution and out-of-distribution settings, demonstrating superior performance across both regimes.
We further deploy \method{} on a physical Franka FR3, achieving $93.3\%$ success across tabletop manipulation tasks, compared with $66.7\%$ for DreamZero~\citep{ye2026dreamzero} and $53.3\%$ for $\pi_{0.5}$~\citep{physicalintelligence2025pi05}.
Controlled experiments show that native action-prior learning consistently outperforms representation-based and latent-action-based approaches across action-label budgets, with larger gains in low-budget regimes.
Together, these results demonstrate that native action-prior learning can effectively leverage observation-only videos to learn action priors, improving generalization and action-label efficiency.

\section{Related Work}
\subsection{World Action Models}
\label{sec:related_wam}
Direct visuomotor policies learn observation-to-action mappings across tasks and embodiments~\citep{brohan2023rt1,octomodelteam2024octo,liu2025rdt1b}, while vision--language--action (VLA) models further leverage semantic priors from large-scale vision--language pre-training~\citep{zitkovich2023rt2,kim2024openvla,black2024pi0,physicalintelligence2025pi05,bjorck2025gr00t,geminiroboticsteam2025gemini,xiaomiRobotics2026xr1,zhang2026hy,shin2026dual,zhang2026foreact,yang2026chain}.
Their primary objective, however, remains conditional action prediction, leaving environment dynamics to be learned implicitly.
With recent advances in video modeling~\citep{an2026onestory,qiao2026track2view,an2026vggrpo}, world action models (WAMs) instead explicitly couple action learning with predictive modeling of future visual evolution.
Existing WAMs can be broadly categorized into cascaded and joint architectures. 
\textit{Cascaded WAMs} first extract predictive information from a video or world model and then use it to drive a separate action module~\citep{du2023unipi,black2024susie,ko2024avdc}. 
Some methods explicitly generate future observations, visual subgoals, or motion cues before action prediction~\citep{bharadhwaj2024gen2act,zhou2024robodreamer,chen2026flowwam,li2026structured,lou2026mask,wang2026mvista}, while others condition policies on predictive intermediate representations without explicit pixel generation~\citep{hu2025vpp,zhang2026imagewam,yan2026s,su2026world,luo2026being}.
\textit{Joint WAMs} model future observations and actions within a shared architecture, enabling direct interaction between visual and action representations~\citep{zhu2025uwm,li2025uva,kim2026cosmospolicy,ye2026gigaworld,huang2026learning,bi2026motus2}. 
DreamZero and LingBot-VA~\citep{ye2026dreamzero,li2026lingbotva} jointly perform future visual prediction and action generation during closed-loop control, whereas Fast-WAM~\citep{yuan2026fastwam} retains future-video prediction as training-time supervision while performing action-only inference. 
Dyna-2~\citep{dynarobotics2026dyna2} further shows that adding human videos through co-training improves cross-embodiment generalization.
\method{} follows the joint WAM paradigm with action-only inference, while enabling the action-generating expert to be directly pretrained from observation-only videos.

\subsection{Learning from Observation-Only Videos} \label{sec:related_control_representations} 
Observation-only videos provide scalable supervision for physical interactions and state transitions without action annotations. 
Existing robot learning approaches primarily leverage them through two paradigms. 
\textit{Vision representation learning} learns visual or predictive representations from videos and uses them during downstream policy learning~\citep{sun2026vla,lin2026jepa,feng2026spatial,nair2023r3m,goswami2025world,zhang2025unijepa,zhang2026diffuview,liang2026bootstrap,ren2026videoworld}. 
VPP~\citep{hu2025vpp} conditions policies on predictive video features, and Masquerade~\citep{lepert2025masquerade} learns visual representations through future robot-keypoint prediction.
V-JEPA~2 and JEPA-VLA~\citep{assran2025vjepa2,miao2026jepa} use predictive video representations for planning and VLA learning. 
Such approaches retain video-derived knowledge in visual representations whose connection to executable actions is established only during downstream policy training, resulting in an indirect representation-to-control transfer.
\textit{Latent-action learning} instead infers action-like variables from visual transitions as intermediate supervision for policy learning~\citep{bruce2024genie,chen2025moto,nikulin2025laom,bu2025univla,tharwat2025lawm,chen2026lawam,liu2026lara,lin2026pixels,wanmulti,lee2026mvp,luo2026joint}. 
LAPO~\citep{schmidt2024lapo} learns latent actions with a latent-to-action decoder, while LAPA~\citep{ye2025lapa} pretrains a VLA through discrete latent-action prediction before grounding them to robot actions. 
Motus~\citep{bi2025motus} incorporates optical-flow motion cues, while RepWAM and LingBot-VA~2.0~\citep{wang2026repwam,zhang2026lingbotva2} learn transition-oriented action representations in semantic visual spaces. 
However, these methods introduce an additional latent-action modeling stage and require architectural or optimization constraints to prevent shortcut solutions~\citep{gao2026dreamdojo,huangcross,yang2026learning,wang2026repwam,zhang2026lingbotva2,chen2026learning}, 
which can constrain the learned latent-action space and the supervision it provides to the policy.
In contrast, \method{} introduces \emph{native action-prior learning} to learn transition priors directly within the action policy, providing a direct path from scalable observation-only videos to action policy pre-training.

\section{Method}
\label{sec:method}

We present \method{}, a world action model that learns action-relevant priors directly within the robot policy from observation-only videos, without relying on an intermediate interface.
Our method is illustrated in Fig.~\ref{fig:main}.
Sec.~\ref{sec:method_preliminaries} introduces the problem formulation and model architecture.
Sec.~\ref{sec:method_pretraining} presents native action-prior pre-training from observation-only videos, and Sec.~\ref{sec:method_posttraining} describes downstream post-training and inference.

\begin{figure}[t]
    \centering
    \includegraphics[width=\linewidth]{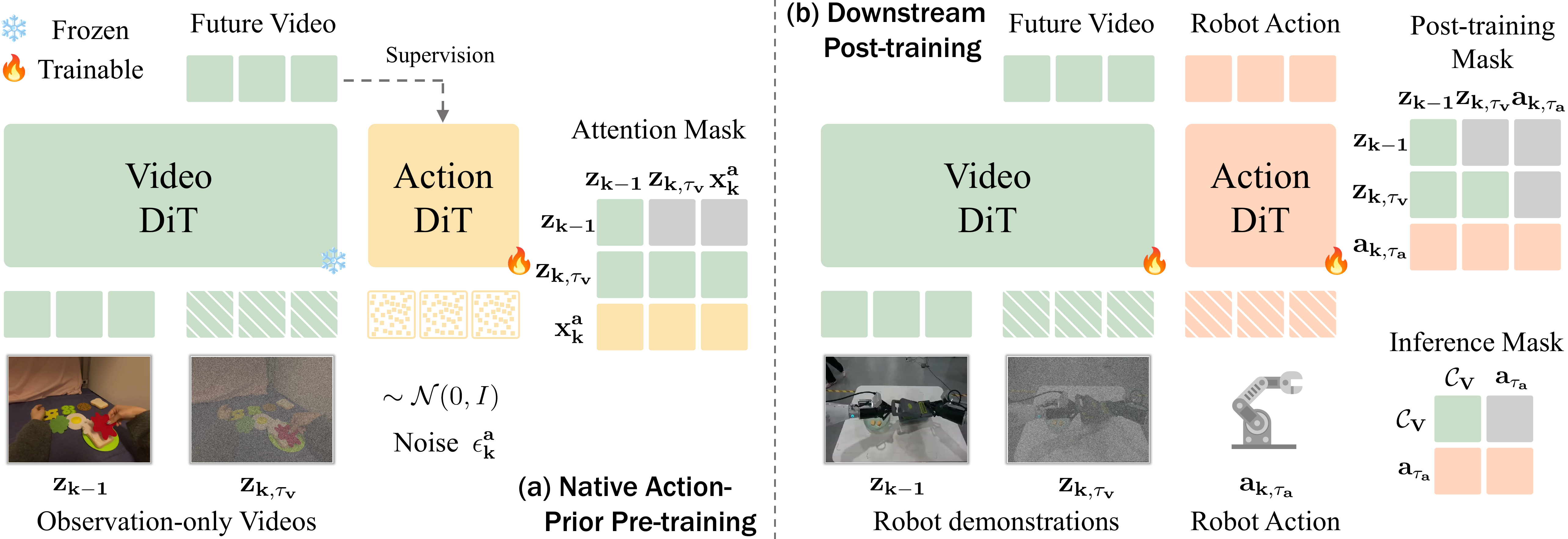}
\caption{\textbf{Overview of \method{}.}
During native action-prior pre-training, the Action-DiT is optimized solely through future-video supervision, learning action-relevant priors directly from observation-only videos.
During post-training, the Video-DiT and Action-DiT are jointly optimized on action-labeled robot demonstrations using video and action flow-matching objectives, adapting the pretrained Action-DiT to continuous robot control.
At inference, \method{} denoises only robot actions without explicitly generating future videos, enabling efficient closed-loop control.
}
    \label{fig:main}
\end{figure}

\subsection{Preliminaries}
\label{sec:method_preliminaries}
\paragraph{Problem formulation.}
Let $o_t$ denote the visual observation at time $t$, $l$ the language instruction, and $q_t$ the current proprioceptive state.
A robot policy $\mathcal{P}_\theta$ predicts a continuous action chunk $a_{t:t+H-1}\in\mathbb{R}^{H\times d_a}$ over a horizon of $H$ steps according to
\begin{equation}
    \mathcal{P}_\theta
    \left(
        a_{t:t+H-1}
        \mid
        o_t,l,q_t
    \right).
    \label{eq:policy_formulation}
\end{equation}
We consider two sources of training data.
The observation-only dataset $\mathcal{D}_v=\{(o_{0:T},l)\}$ contains videos and language descriptions without action annotations.
The action-labeled robot dataset $\mathcal{D}_r=\{(o_{0:T},l,q,a_{0:T-1})\}$ additionally provides proprioceptive states $q$ and continuous robot actions $a_{0:T-1}$.
A frozen spatiotemporal VAE $E$~\citep{wan2025} encodes each video clip into latent visual states $z_{0:K}=E(o_{0:T})$, where each adjacent pair $(z_{k-1},z_k)$ defines a visual state transition indexed by $k$.
Our goal is to leverage the abundant action-free transitions in $\mathcal{D}_v$ to pretrain the action policy before adapting it to embodiment-specific control using $\mathcal{D}_r$.

\paragraph{World action model architecture.}
We instantiate $\mathcal{P}_\theta$ as a Mixture-of-Transformers (MoT)~\citep{liang2024mixture} comprising a pretrained Video-DiT with parameters $\theta_V$ and an Action-DiT with parameters $\theta_A$.
The Video-DiT models visual dynamics, while the Action-DiT generates continuous robot actions.
The two experts retain modality-specific parameters while exchanging information through joint attention.
Each expert independently constructs its queries, keys, and values.
Queries from modality $m\in\{V,A\}$ attend jointly to the keys and values from both streams:
\begin{equation}
    H_m
    =
    \operatorname{Attn}\!\left(
        Q_m,
        [K_V;K_A],
        [V_V;V_A]
    \right).
    \label{eq:joint_attention}
\end{equation}
Together, the two experts form a unified architecture for modeling visual dynamics and robot actions.

\subsection{Native Action-Prior Pre-training}
\label{sec:method_pretraining}
Observation-only videos capture visual state transitions that reveal the consequences of physical interactions, providing a rich and scalable source for learning action-relevant priors.
Our goal is to improve WAM training by introducing a pre-training stage that learns action priors from such videos, while keeping the model architecture consistent with downstream WAM training and making only minimal changes.
To this end, we propose \emph{native action-prior pre-training}, which directly exploits transitions in $\mathcal{D}_v$ to pretrain the action policy without introducing an intermediate interface.

Concretely, we retain the same MoT architecture used for downstream WAM training.
During pre-training, action tokens are not supervised by robot actions; instead, we structure joint attention around visual transitions and optimize the Action-DiT solely through future-video supervision.
This enables the action policy itself to acquire action priors directly from observation-only videos.

\paragraph{Inverse-forward dynamics modeling with attention.}
We motivate action-prior learning through an inverse--forward dynamics factorization~\citep{schmidt2024lapo,ye2025lapa}.
Given two adjacent states $z_{k-1}$ and $z_k$, inverse dynamics extracts a transition-dependent variable $r_k$, while forward dynamics uses $r_k$ together with the preceding state to predict the future state:
\begin{equation}
    r_k
    =
    q(z_{k-1},z_k),
    \qquad
    \hat z_k
    =
    f(z_{k-1},r_k).
    \label{eq:inverse_forward_abstraction}
\end{equation}
Rather than introducing separate inverse and forward models~\citep{cui2024dynamo,wang2026repwam,zhang2026lingbotva2,yang2026learning}, we realize an analogous information flow implicitly within the policy through \emph{transition-structured joint attention}.

For each transition $z_{k-1}\rightarrow z_k$, we construct a segment containing the clean preceding state $z_{k-1}$, a noised version of the future state $z_k$, denoted $z_{k,\tau_v}$ at timestep $\tau_v$, and the corresponding group of Action-DiT tokens $x_k^a$, as illustrated in Fig.~\ref{fig:main}.
Transition segments are isolated by the attention mask (detailed in Appendix~\ref{app:attention_mask}).
Within each segment, clean preceding-state queries attend only to the clean state, whereas noised-future and action queries interact with all three token groups.

Under this transition-structured mask, the Action-DiT attention is given by
\begin{equation}
    H_{A,k}
    =
    \operatorname{Attn}\!\left(
        Q_{A,k},
        [K_{V,k-1};
         K_{V,k,\tau_v};
         K_{A,k}],
        [V_{V,k-1};
         V_{V,k,\tau_v};
         V_{A,k}]
    \right).
    \label{eq:transition_action_attention}
\end{equation}
The Action-DiT jointly accesses the preceding and future visual states and forms an internal representation of their transition.
This information flow is analogous to inverse dynamics in Eq.~\ref{eq:inverse_forward_abstraction}, while keeping the transition representation implicit rather than materializing explicit latent actions.

Conversely, the noised future-state tokens attend to both visual states and the corresponding Action-DiT representations:
\begin{equation}
    H_{V,k,\tau_v}
    =
    \operatorname{Attn}\!\left(
        Q_{V,k,\tau_v},
        [K_{V,k-1};
         K_{V,k,\tau_v};
         K_{A,k}],
        [V_{V,k-1};
         V_{V,k,\tau_v};
         V_{A,k}]
    \right).
    \label{eq:transition_video_attention}
\end{equation}
This interaction parallels forward dynamics in Eq.~\ref{eq:inverse_forward_abstraction}, where future-state prediction depends on both the preceding state and a transition representation.
Here, the Action-DiT accesses the future state only through its noised version $z_{k,\tau_v}$, precluding the trivial shortcut of directly copying the clean prediction target $z_k$~\citep{garrido2026learning,ye2025lapa}.
See Appendix~\ref{app:action_prior_analysis} for formal analysis.
Together, Eqs.~\ref{eq:transition_action_attention} and~\ref{eq:transition_video_attention} realize an implicit inverse--forward interaction within the joint transformer, with the Action-DiT modeling the transition between adjacent visual states.

\paragraph{Pre-training formulation.}
We use standard flow matching for future-video prediction~\citep{lipman2023flowmatching}.
For each target future state $z_k$, we sample $\epsilon_k^v\sim\mathcal{N}(0,I)$ and a flow timestep $\tau_v\in[0,1]$, and construct
\begin{align}
    z_{k,\tau_v}
    &=
    (1-\tau_v)z_k
    +
    \tau_v\epsilon_k^v,
    &
    u_k^v
    &=
    \epsilon_k^v-z_k,
    \label{eq:pretraining_flow}
\end{align}
where $\tau_v=0$ and $\tau_v=1$ correspond to the data and noise endpoints, respectively, and $u_k^v$ is the target velocity.
For the corresponding Action-DiT input, we independently sample $\epsilon_k^a\sim\mathcal{N}(0,I)$ with the action timestep at $\tau_a=1$.

For each transition, the Video-DiT receives $[z_{k-1},z_{k,\tau_v}]$, while the Action-DiT receives $\epsilon_k^a$.
We freeze the pretrained Video-DiT, $\theta_V=\bar{\theta}_V$, and optimize only the Action-DiT parameters $\theta_A$.
With stream-wise timesteps $\boldsymbol{\tau}=(\tau_v,1)$, the pre-training objective is
\begin{equation}
    \mathcal{L}_{\mathrm{pre}}
    =
    \mathbb{E}\!\left[
        \sum_{k=1}^{K}
        w(\tau)
        \left\|
        v^v_{\bar{\theta}_V,\theta_A}
        \!\left(
            [z_{k-1},z_{k,\tau_v}],
            \epsilon_k^a,
            \boldsymbol{\tau}
            \mid l
        \right)
        -
        u_k^v
        \right\|_2^2
    \right].
    \label{eq:pretraining_objective}
\end{equation}
Through transition-structured joint attention, future-video prediction depends on the Action-DiT representations.
Because the Video-DiT is frozen, minimizing the future-video objective directly updates $\theta_A$ through the joint-attention pathway.
This direct optimization of the Action-DiT from visual transitions constitutes our native action-prior learning, allowing action-relevant priors to emerge without a separate transfer interface.

\subsection{Post-training and Inference}
\label{sec:method_posttraining}

Action-labeled robot demonstrations provide direct supervision for embodiment-specific continuous control.
We post-train on $\mathcal{D}_r$ to ground the action priors learned from observation-only videos to executable robot actions.
We retain the transition-based structure from pre-training while adapting the cross-stream interaction to direct action supervision and efficient inference.

\paragraph{Asymmetric transition attention.}
During native action-prior pre-training, future-video tokens attend to Action-DiT representations through Eq.~\ref{eq:transition_video_attention}, allowing the video objective to propagate supervision into the Action-DiT.
During post-training, action labels become available and the Action-DiT receives direct supervision.
We therefore make the visual stream independent of the action stream while retaining visual conditioning for the Action-DiT.
This asymmetric interaction enables the visual representations to be cached and reused throughout action denoising at inference.

Specifically, as shown in Fig.~\ref{fig:main}, Video-DiT queries attend only to the visual stream:
\begin{equation}
    H_V
    =
    \operatorname{Attn}\!\left(
        Q_V,
        K_V,
        V_V
    \right),
    \label{eq:post_video_attention}
\end{equation}
while Action-DiT queries continue to attend to both visual and action streams:
\begin{equation}
    H_A
    =
    \operatorname{Attn}\!\left(
        Q_A,
        [K_V;K_A],
        [V_V;V_A]
    \right).
    \label{eq:post_action_attention}
\end{equation}

\paragraph{Post-training formulation.}
We jointly train the visual and action streams using standard flow matching.
For each transition $z_{k-1}\rightarrow z_k$ and its corresponding continuous action chunk $a_k$, we independently sample $\epsilon_k^v,\epsilon_k^a\sim\mathcal{N}(0,I)$ and flow timesteps $\tau_v,\tau_a\in[0,1]$, and construct
\begin{align}
    z_{k,\tau_v}
    &=
    (1-\tau_v)z_k
    +
    \tau_v\epsilon_k^v,
    &
    u_k^v
    &=
    \epsilon_k^v-z_k,
    \nonumber\\
    a_{k,\tau_a}
    &=
    (1-\tau_a)a_k
    +
    \tau_a\epsilon_k^a,
    &
    u_k^a
    &=
    \epsilon_k^a-a_k.
    \label{eq:posttraining_flow}
\end{align}
With $\boldsymbol{\tau}=(\tau_v,\tau_a)$, the post-training objective is
$
    \mathcal{L}_{\mathrm{post}}
    =
    \lambda_v\mathcal{L}_{\mathrm{FM}}^v
    +
    \lambda_a\mathcal{L}_{\mathrm{FM}}^a,
$
where
\begin{equation}
    \mathcal{L}_{\mathrm{FM}}^m
    =
    \mathbb{E}\!\left[
        \sum_{k=1}^{K}
        w_m(\tau_m)
        \left\|
        v_{\theta_V,\theta_A}^m
        \!\left(
            [z_{k-1},z_{k,\tau_v}],
            a_{k,\tau_a},
            \boldsymbol{\tau}
            \mid l,q
        \right)
        -
        u_k^m
        \right\|_2^2
    \right],
    \qquad
    m\in\{v,a\}.
    \label{eq:posttraining_objective}
\end{equation}
Here, $\lambda_v$ and $\lambda_a$ weight the visual and action objectives, respectively.
The video objective adapts the visual expert to downstream robot observations, while the action objective grounds the pretrained Action-DiT to embodiment-specific continuous control.

\paragraph{Inference.}
Asymmetric attention makes the Video-DiT independent of the evolving action sample.
We initialize the future-video and action streams with Gaussian noise and evaluate them at the first denoising step.
During this forward pass, we cache the Video-DiT keys and values at each layer $\ell$:
\begin{equation}
    \mathcal{C}_V
    =
    \{K_\ell^V,V_\ell^V\}_{\ell=1}^{L}.
    \label{eq:video_kv_cache}
\end{equation}
Subsequent denoising steps integrate only the action flow while reusing the visual cache $\mathcal{C}_V$:
\begin{equation}
    \frac{d a_{\tau_a}}{d\tau_a}
    =
    v_{\theta_A}^a
    \!\left(
        a_{\tau_a},
        \tau_a
        \mid
        \mathcal{C}_V,l,q
    \right),
    \qquad
    \tau_a:1\rightarrow0.
    \label{eq:cached_action_generation}
\end{equation}
Thus, after the initial joint forward pass, iterative denoising runs only through the Action-DiT, enabling efficient closed-loop control without explicitly generating future videos.

\section{Experiments}
\label{sec:experiments}

We evaluate \method{} on standard simulation benchmarks under both \emph{in-domain} (ID) and \emph{out-of-domain} (OOD) settings, and further validate its deployment on a physical robot.

\subsection{Experimental Setup}
\label{sec:experimental_setup}
\method{} initializes the Video-DiT from Wan2.2-5B~\citep{wan2025} and uses a 1B-parameter Action-DiT with hidden dimension $1{,}024$~\citep{li2026lingbotva,yuan2026fastwam}.
Native action-prior pre-training uses observation-only videos from Open X-Embodiment~\citep{o2024open}, AgiBotWorld~\citep{bu2025agibot}, and EgoDex~\citep{hoque2025egodex}. 
For downstream adaptation, we optimize the video and action streams with $\lambda_v=\lambda_a=1$. 
More details are provided in Appendix~\ref{app:implementation_details}.

\begin{table}[!t]
    \centering
    \footnotesize
    \setlength{\tabcolsep}{2pt}
    \renewcommand{\arraystretch}{1.08}
    \caption{\textbf{Results on LIBERO/LIBERO-Plus (\textit{left}) and RoboTwin~2.0 (\textit{right}).}
    We report average success rate (\%).
    LIBERO and RoboTwin Clean measure in-domain performance, while LIBERO-Plus and RoboTwin Random evaluate out-of-domain generalization.
    Best results are shown in \textbf{bold}.}
    \label{tab:main_results}
    \vspace{-6pt}
    \begin{minipage}[t]{0.485\linewidth}
        \vspace{0pt}
        \centering
        \begin{tabularx}{\linewidth}{
            l
            >{\centering\arraybackslash}X
            >{\centering\arraybackslash}X
        }
            \toprule
            Method & LIBERO & LIBERO-Plus \\
            \midrule
            \rowcolor{gray!10}
            \multicolumn{3}{@{}l}{\emph{Direct action policies}} \\
            $\pi_0$~\citeyearpar{black2024pi0} & 94.1 & 53.6 \\
            $\pi_0$-FAST~\citeyearpar{pertsch2025fast} & 85.5 & 61.6 \\
            StarVLA-$\alpha$~\citeyearpar{ye2026starvla} & 96.5 & 77.0 \\
            $\pi_{0.5}$~\citeyearpar{physicalintelligence2025pi05} & 96.9 & 77.4 \\
            ABot-M0~\citeyearpar{yang2026abot} & 98.6 & 80.5 \\

            \addlinespace[2pt]
            \rowcolor{gray!10}
            \multicolumn{3}{@{}l}{\emph{World action models}} \\
            \addlinespace[2pt]
            JEPA-VLA~\citeyearpar{miao2026jepa} & 96.4 & 25.6 \\
            Fast-WAM~\citeyearpar{yuan2026fastwam} & 97.6 & 51.5 \\
            Image-WAM~\citeyearpar{zhang2026imagewam} & 98.4 & 83.1 \\
            Being-H0.7~\citeyearpar{luo2026being} & \textbf{99.2} & 82.1 \\

            \addlinespace[2pt]
            \rowcolor{row}
            \textbf{\method{}} & 99.0 & \textbf{83.5} \\
            \bottomrule
        \end{tabularx}
    \end{minipage}
    \hfill
    \begin{minipage}[t]{0.485\linewidth}
        \vspace{0pt}
        \centering

        \begin{tabularx}{\linewidth}{
            l
            >{\centering\arraybackslash}X
            >{\centering\arraybackslash}X
        }
            \toprule
            Method & Clean & Random \\
            \midrule
            \rowcolor{gray!10}
            \multicolumn{3}{@{}l}{\emph{Direct action policies}} \\
            DP~\citeyearpar{chi2025diffusion} & 28.0 & 0.6 \\
            RDT~\citeyearpar{liu2025rdt1b} & 34.5 & 13.7 \\
            $\pi_0$~\citeyearpar{black2024pi0} & 46.4 & 16.3 \\
            UP-VLA~\citeyearpar{upvla25} & 52.9 & 15.2 \\

            \addlinespace[2pt]
            \rowcolor{gray!10}
            \multicolumn{3}{@{}l}{\emph{World action models}} \\
            \addlinespace[2pt]
            Fast-WAM~\citeyearpar{yuan2026fastwam} & 71.9 & 6.3 \\
            BagelVLA~\citeyearpar{hu2026bagelvla} & 75.3 & 20.5 \\
            HALO~\citeyearpar{shou2026halo} & 80.5 & 26.4 \\
            Image-WAM~\citeyearpar{zhang2026imagewam} & 85.0 & 37.6 \\
            MV-WAM~\citeyearpar{chen2026mv} & 84.0 & 55.7 \\

            \addlinespace[2pt]
            \rowcolor{row}
            \textbf{\method{}} & \textbf{88.5} & \textbf{73.6} \\
            \bottomrule
        \end{tabularx}
    \end{minipage}
    \vspace{1pt}
\end{table}

\begin{figure}[!t]
    \centering
    \includegraphics[width=\linewidth]{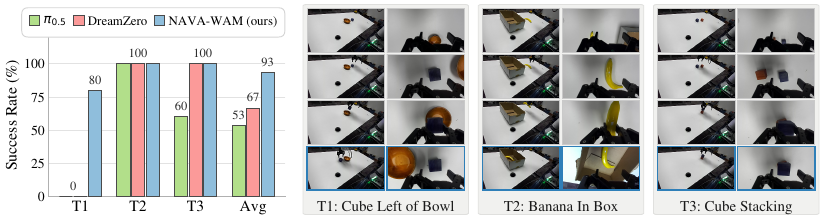}
    \captionsetup{skip=6pt}
    \caption{\textbf{Real-robot evaluation on a Franka FR3.}
    \emph{Left}: success rate over five trials per task and average success over all $15$ trials.
    \emph{Right}: \method{} rollouts for the three evaluation tasks, shown from third-person and wrist-camera views.
    \method{} achieves $93.3\%$ average success and correctly follows the spatial relation in T1, where both baselines fail across all trials.}
    \vspace{-6pt}
    \label{fig:real_robot}
\end{figure}

\textbf{Simulation evaluation.}
We evaluate on two simulation benchmark families.
\textit{LIBERO and LIBERO-Plus.}
LIBERO~\citep{liu2023libero} contains four 10-task suites, while LIBERO-Plus~\citep{fei25libero-plus} extends them with seven distribution shifts.
Following Zhang et al.~\citeyearpar{zhang2026imagewam} and Yuan et al.~\citeyearpar{yuan2026fastwam}, we train on 50 LIBERO demonstrations per task and evaluate on LIBERO (ID) and LIBERO-Plus (OOD).
\textit{RoboTwin~2.0.}
We train exclusively on 50 Clean-domain demonstrations per task from RoboTwin~2.0~\citep{chen2025robotwin} and use no Random-domain demonstrations.
Following Chen et al.~\citeyearpar{chen2026mv}, we evaluate in both Clean and Random domains, treating Random as the OOD setting.

\textbf{Real-robot evaluation.}
We additionally deploy \method{} post-trained on DROID~\citep{khazatsky2024droid} on a physical Franka FR3 arm \textit{without task-specific fine-tuning}.
We evaluate three tabletop manipulation tasks with five trials each and compare against $\pi_{0.5}$~\citep{physicalintelligence2025pi05} and the DreamZero-DROID policy~\citep{ye2026dreamzero}.
Additional details are in Appendix~\ref{app:real_robot}.

\begin{table}[!t]
\centering
\caption{\textbf{Ablation studies on RoboTwin~2.0.}
(a) Comparison of observation-only pre-training methods under varying action-label budgets.
(b) Action relevance of learned representations, measured by ridge-regression $R^2$ to ground-truth actions under different probe-data budgets.
(c) Inference compute per action chunk and Random-domain performance.
Best results are shown in \textbf{bold}.}
\label{tab:ablation}
\vspace{-3pt}
\begin{subtable}[t]{0.33\linewidth}
    \centering
    \subcaption{Pre-training \& label efficiency}
    \label{tab:data_efficiency}
    \vspace{-3pt}

    \setlength{\tabcolsep}{3pt}
    \renewcommand{\arraystretch}{1.1}
    \resizebox{\linewidth}{!}{%
    \begin{tabular}{@{}lccc@{}}
        \toprule
        & \multicolumn{3}{c}{Demos / task} \\[-2pt]
        \cmidrule(lr){2-4}
        \noalign{\vskip -2pt}
        Pre-training method & $10$ & $25$ & $50$ \\
        \midrule
        No pre-training       & 60.5 & 67.6 & 69.7 \\
        Representation-based  & 58.2 & 60.6 & 68.1 \\
        Latent-action-based   & 61.6 & 65.0 & 71.2 \\
        \rowcolor{row}
        \textbf{Ours}         & \textbf{66.5} & \textbf{72.0} & \textbf{73.6} \\
        \bottomrule
    \end{tabular}}
\end{subtable}
\hfill
\begin{subtable}[t]{0.36\linewidth}
    \centering
    \subcaption{Action relevance}
    \label{tab:linear_probe}
    \vspace{-3pt}

    \setlength{\tabcolsep}{3pt}
    \renewcommand{\arraystretch}{1.148}
    \resizebox{\linewidth}{!}{%
    \begin{tabular}{@{}lccc@{}}
        \toprule
        & \multicolumn{3}{c}{Probe data} \\[-1.6pt]
        \cmidrule(lr){2-4}
        \noalign{\vskip -1.6pt}
        Feature source & $10\%$ & $50\%$ & $100\%$ \\
        \midrule
        Representation-based
            & 0.080 & 0.157 & 0.153 \\
        $\text{Latent-action}_{~\text{(CoMo)}}$
            & 0.088 & 0.192 & 0.203 \\
        $\text{Latent-action}_{~\text{(DynaMo)}}$
            & 0.052 & 0.143 & 0.154 \\
        \rowcolor{row}
        \textbf{Ours}
            & \textbf{0.247} & \textbf{0.320} & \textbf{0.329} \\
        \bottomrule
    \end{tabular}}
\end{subtable}
\hfill
\begin{subtable}[t]{0.26\linewidth}
    \centering
    \subcaption{Inference efficiency}
    \label{tab:efficiency_optimized}
    \vspace{-3pt}

    \setlength{\tabcolsep}{3pt}
    \renewcommand{\arraystretch}{1.45}
    \resizebox{\linewidth}{!}{%
    \begin{tabular}{@{}lcc@{}}
        \toprule
        Method & TFLOPs & Random \\
        \midrule
        $\text{Fast-WAM}$
            & \textbf{3.56} & 6.3 \\
        $\text{HALO}$
            & 67.69 & 26.4 \\
        $\text{Image-WAM}$
            & 4.09 & 37.6 \\
        \rowcolor{row}
        \textbf{Ours}
            & 4.30 & \textbf{73.6} \\
        \bottomrule
    \end{tabular}}
\end{subtable}
\end{table}

\begin{figure}[!t]
    \centering
    \includegraphics[width=\linewidth]{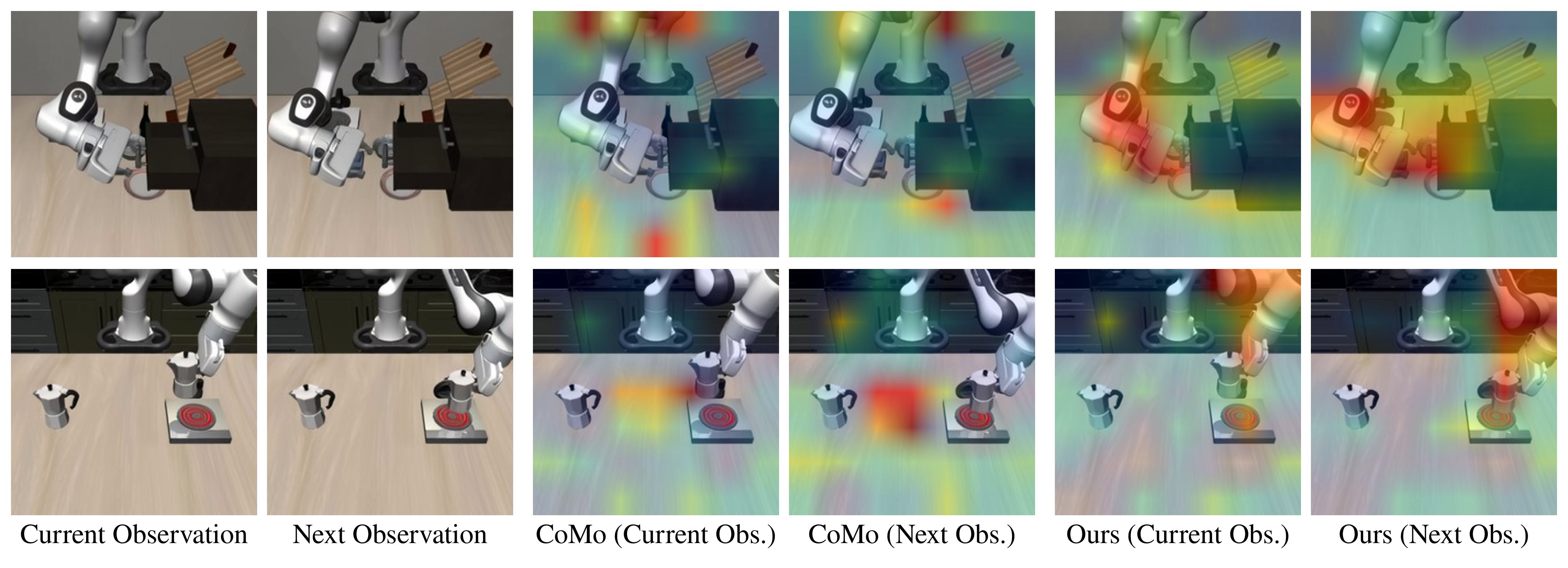}
    \captionsetup{skip=1pt}
\caption{\textbf{Attention of learned action representations.}
We visualize attention from action representations to visual tokens on held-out data, overlaid on the \textit{current} and \textit{next} input observations.
Our pretrained Action-DiT attends to interaction-relevant regions and aligns with observed motion, while the CoMo latent-action baseline shows more diffuse attention over static regions.}
\label{fig:attention_localization}
\end{figure}

\subsection{Main Results}
\label{sec:benchmark_results}

\textbf{Simulation results.}
Tab.~\ref{tab:main_results} shows that \method{} achieves $99.0\%$ success on LIBERO and $83.5\%$ on LIBERO-Plus, maintaining strong performance under substantial distribution shifts.
On RoboTwin~2.0, \method{} reaches $88.5\%$ on Clean and $73.6\%$ on Random.
Compared with the strongest baseline, \method{} improves success by $3.5$ points on Clean and $17.9$ points on Random, while 
exhibiting a much smaller Clean-to-Random performance drop.
Together, these results show that native action-prior learning preserves strong in-domain control while substantially improving generalization under visual and environmental distribution shifts.

\textbf{Real-robot results.}
As shown in Fig.~\ref{fig:real_robot}, \method{} succeeds in $14/15$ trials ($93.3\%$), compared with $10/15$ ($66.7\%$) for DreamZero and $8/15$ ($53.3\%$) for $\pi_{0.5}$.
The largest difference occurs on T1, which requires placing a cube on the \emph{left side} of a bowl: \method{} succeeds in $4/5$ trials, whereas both baselines fail in all five trials.
These results complement the simulation benchmarks and demonstrate that the learned action prior transfers effectively to physical robot deployment.

\begin{figure}[!t]
    \centering
    \includegraphics[width=\linewidth]{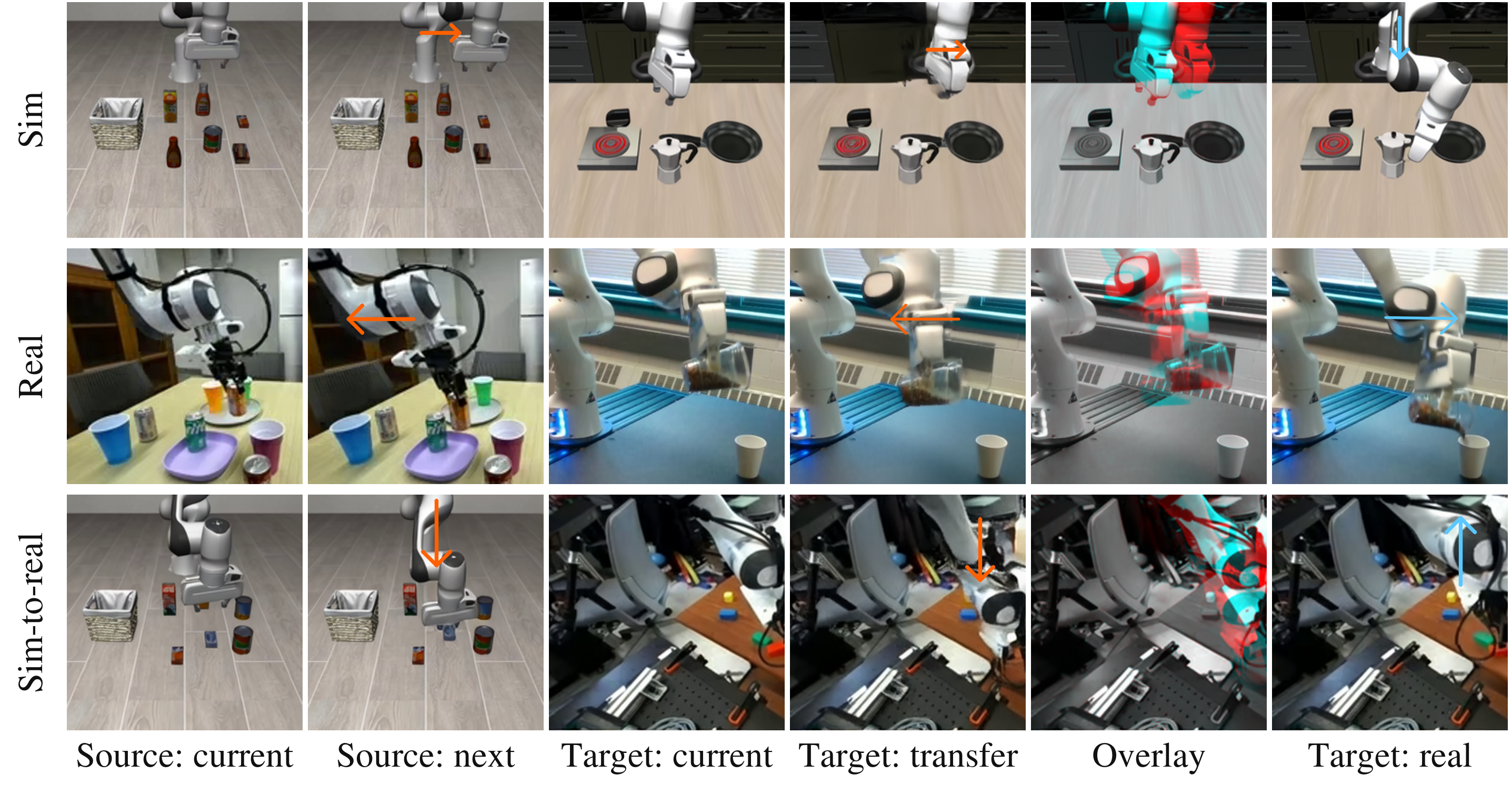}
    \captionsetup{skip=3pt}
    \caption{\textbf{Qualitative visualization of motion transfer across domains.}
    Each example transfers the transition encoded by a source pair (\textit{Source: current} $\rightarrow$ \textit{Source: next}) to a different target observation.
    The source transition specifies the underlying motion cue (orange arrows), which is applied to \textit{Target: current} to obtain \textit{Target: transfer}.
    The \textit{Overlay} visualizes the transferred state relative to the current target observation, while \textit{Target: real} shows the corresponding real transition (blue arrows).
    Across simulation, real-world, and sim-to-real examples, the transferred states follow the source motion despite substantial changes in visual appearance and scene configuration.}
    \label{fig:motion_transfer}
\label{fig:motion_transfer}
\end{figure}

\subsection{Ablation Studies}
\label{sec:ablation_studies}

We conduct ablations on RoboTwin~2.0 to analyze native action-prior learning and its impact on downstream control.
Unless otherwise specified, all variants follow the same training and evaluation protocol and differ only in the component under study.

\begin{wrapfigure}{r}{0.35\textwidth}
\centering
\vspace{-1.5em}
\includegraphics[width=\linewidth]{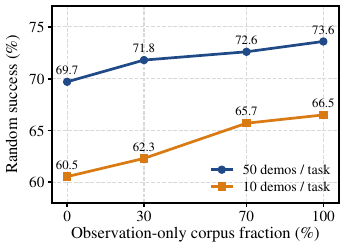}
\vspace{-2em}
\caption{\textbf{Scaling observation-only pre-training data.} Success across data scales at different action-label budgets.}
\label{fig:corpus_scale}
\vspace{-3.0em}
\end{wrapfigure}

\paragraph{Observation-only pre-training method and action-label efficiency.}
We compare different methods for leveraging observation-only videos under the same downstream protocol.
As shown in Tab.~\ref{tab:data_efficiency}, \method{} consistently outperforms representation-based and latent-action-based pre-training~\citep{yang2026learning} across all action-label budgets. 
With 25 demonstrations per task, it achieves 72.0\% success versus 60.6\% and 65.0\%, respectively. Moreover, NAVA-WAM with only 25 demonstrations surpasses all baselines trained with 50, demonstrating both more effective observation-only pre-training and improved action-label efficiency.

\paragraph{Observation-only data scale.}
We study how scaling observation-only pre-training data affects downstream generalization.
As shown in Fig.~\ref{fig:corpus_scale}, Random-domain success consistently improves with corpus size under both action-label budgets,
demonstrating that native action-prior learning benefits from more observation-only data.

\paragraph{Action relevance of learned representations.}
We probe the action relevance of representations learned from observation-only videos: Video-DiT features from vision representation learning, inferred latent actions from latent-action learning~\citep{cui2024dynamo,yang2026learning}, and Action-DiT features from native action-prior learning.
In Tab.~\ref{tab:linear_probe}, \method{} achieves the highest $R^2$ across all probe-data budgets.
Fig.~\ref{fig:attention_localization} further provides complementary qualitative evidence: our Action-DiT focuses on interaction-relevant regions such as the gripper and manipulated objects, whereas the latent-action baseline~\citep{yang2026learning} exhibits more diffuse attention over static regions.
Together, these results indicate that native action-prior learning encodes transition-relevant action information directly within the Action-DiT.

\paragraph{Qualitative motion transfer.}
Fig.~\ref{fig:motion_transfer} analyzes the action prior learned by our pretrained Action-DiT.
Given a source transition, we extract its Action-DiT activations and transfer them to a different target observation to predict the corresponding next state~\citep{garrido2026learning,todd2024function}.
Across simulation, real-world, and sim-to-real examples, the transferred states reproduce the source motion despite substantial changes in appearance and scene configuration.
These results show that native action-prior learning captures transferable action priors beyond the specific visual content of the source transition.

\paragraph{Inference efficiency.}
Tab.~\ref{tab:efficiency_optimized} compares the inference compute across methods.
\method{} requires $4.30$~TFLOPs, comparable to Image-WAM ($4.09$~TFLOPs) and Fast-WAM ($3.56$~TFLOPs), while achieving substantially higher Random-domain success at $73.6\%$.
This demonstrates that our model achieves strong out-of-domain generalization with low inference overhead.

\section{Conclusion}
We present \method{}, which introduces \emph{native action-prior learning} for world action models, directly pre-training the action policy from observation-only videos without relying on an intermediate interface.
By learning action-relevant priors from visual transitions and adapting them to robot control with action-labeled demonstrations, \method{} unifies action-free pre-training and downstream policy learning within a single model.
Experiments demonstrate strong generalization in simulation and real-robot deployment, with efficient action-only inference.
Overall, our results suggest that directly learning action priors from scalable observation-only videos offers a simple and effective path toward scaling world action models beyond costly action-labeled robot data.

\section*{Acknowledgements}
We would like to thank Tian Xie (Meta AI), Shanlin Sun (Meta AI), and Hyojun Go (ETH Zurich) for their constructive feedback on this project.
Zhaochong An and Serge Belongie are supported by funding from the Pioneer Centre for AI, DNRF grant number P1.

\bibliographystyle{assets/plainnat}
\bibliography{paper}

\clearpage
\newpage
\beginappendix

\section{Theoretical Analysis of Native Action-Prior Learning}
\label{app:action_prior_analysis}
We provide a simplified theoretical analysis of why the Action-DiT pathway captures transition information and how future-state corruption prevents a direct target copying shortcut.
Following the notation in Sec.~\ref{sec:method_pretraining}, consider a transition from the preceding visual state $z_{k-1}$ to the future state $z_k$.
During pre-training, the Action-DiT does not directly observe the clean future state $z_k$, but instead accesses its noised version
\begin{equation}
    z_{k,\tau}
    =
    (1-\tau)z_k
    +
    \tau\epsilon_k^v,
    \qquad
    \epsilon_k^v\sim\mathcal{N}(0,I),
    \label{eq:app_corrupted_future}
\end{equation}
as defined in Eq.~\ref{eq:pretraining_flow}.
We abstract the transition representation formed by the Action-DiT as
\begin{equation}
    r_k
    =
    E(z_{k-1},z_{k,\tau}),
    \label{eq:app_transition_representation}
\end{equation}
and consider predicting the clean future state from the preceding state and this representation,
\begin{equation}
    \hat z_k
    =
    D(z_{k-1},r_k).
    \label{eq:app_future_predictor}
\end{equation}
This abstraction isolates the information carried through the Action-DiT pathway while suppressing the architectural details of the joint transformer.

\paragraph{The action pathway captures transition information.}
We first characterize the predictive value of the transition representation $r_k$.
Define the Bayes-optimal prediction errors
\begin{align}
    \mathcal{L}_{z}
    &=
    \mathbb{E}\!\left[
        \left\|
        z_k-\mathbb{E}[z_k\mid z_{k-1}]
        \right\|_2^2
    \right],
    \label{eq:app_prediction_error_state}
    \\
    \mathcal{L}_{z,r}
    &=
    \mathbb{E}\!\left[
        \left\|
        z_k-\mathbb{E}[z_k\mid z_{k-1},r_k]
        \right\|_2^2
    \right].
    \label{eq:app_prediction_error_transition}
\end{align}

\begin{proposition}[Predictive value of the transition representation]
\label{prop:transition_information}
The reduction in Bayes-optimal future-state prediction error from conditioning on $r_k$ is
\begin{equation}
    \mathcal{L}_{z}
    -
    \mathcal{L}_{z,r}
    =
    \mathbb{E}\!\left[
        \left\|
        \mathbb{E}[z_k\mid z_{k-1},r_k]
        -
        \mathbb{E}[z_k\mid z_{k-1}]
        \right\|_2^2
    \right]
    \geq 0.
    \label{eq:app_prediction_error_reduction}
\end{equation}
Thus, $r_k$ is useful to the extent that it captures information about the future state beyond what is already predictable from the preceding state $z_{k-1}$.
\end{proposition}

\begin{proof}
Let
\begin{equation}
    m_{k-1}
    =
    \mathbb{E}[z_k\mid z_{k-1}],
    \qquad
    m_{k-1,r}
    =
    \mathbb{E}[z_k\mid z_{k-1},r_k].
    \label{eq:app_conditional_means}
\end{equation}
We decompose
\begin{equation}
    z_k-m_{k-1}
    =
    (z_k-m_{k-1,r})
    +
    (m_{k-1,r}-m_{k-1}).
    \label{eq:app_error_decomposition}
\end{equation}
The two terms are orthogonal in expectation because
\begin{equation}
    \mathbb{E}[z_k-m_{k-1,r}\mid z_{k-1},r_k]
    =
    0.
    \label{eq:app_conditional_residual}
\end{equation}
Therefore,
\begin{equation}
    \mathbb{E}\|z_k-m_{k-1}\|_2^2
    =
    \mathbb{E}\|z_k-m_{k-1,r}\|_2^2
    +
    \mathbb{E}\|m_{k-1,r}-m_{k-1}\|_2^2,
    \label{eq:app_orthogonal_decomposition}
\end{equation}
which yields Eq.~\ref{eq:app_prediction_error_reduction}.
\end{proof}

\paragraph{Future corruption prevents a target-copying shortcut.}
If the Action-DiT directly observed the clean future state $z_k$, the prediction objective would admit the degenerate shortcut
\begin{equation}
    r_k=z_k,
    \qquad
    D(z_{k-1},r_k)=r_k,
    \label{eq:app_target_copying}
\end{equation}
which achieves zero prediction error by directly transmitting the prediction target, without requiring the action pathway to capture a meaningful transition representation.
Instead, our attention construction exposes the Action-DiT only to the Gaussian-corrupted future state $z_{k,\tau}$.

\begin{proposition}[Gaussian corruption prevents exact target copying]
\label{prop:corruption_copying}
Assume
\begin{equation}
    z_k\mid z_{k-1}
    \sim
    \mathcal{N}\!\left(
        \mu(z_{k-1}),
        \Sigma
    \right),
    \qquad
    \Sigma\succ0,
    \label{eq:app_gaussian_assumption}
\end{equation}
and let $r_k=E(z_{k-1},z_{k,\tau})$ be any representation computed from the corrupted future state in Eq.~\ref{eq:app_corrupted_future}.
Then
\begin{equation}
    \mathcal{L}_{z,r}
    \geq
    \mathbb{E}\!\left[
        \left\|
        z_k
        -
        \mathbb{E}[z_k\mid z_{k-1},z_{k,\tau}]
        \right\|_2^2
    \right].
    \label{eq:app_corruption_lower_bound}
\end{equation}
Moreover, for any $\tau>0$,
\begin{equation}
    \operatorname{Cov}
    (z_k\mid z_{k-1},z_{k,\tau})
    =
    \left(
        \Sigma^{-1}
        +
        \frac{(1-\tau)^2}{\tau^2}I
    \right)^{-1},
    \label{eq:app_conditional_covariance}
\end{equation}
and therefore
\begin{equation}
    \mathcal{L}_{z,r}
    \geq
    \operatorname{tr}
    \left[
        \left(
            \Sigma^{-1}
            +
            \frac{(1-\tau)^2}{\tau^2}I
        \right)^{-1}
    \right]
    >0.
    \label{eq:app_positive_prediction_error}
\end{equation}
Hence, at any non-zero corruption level, the clean future state cannot be transmitted perfectly through the transition representation $r_k$.
\end{proposition}

\begin{proof}
Because $r_k$ is computed from $(z_{k-1},z_{k,\tau})$, conditioning on $(z_{k-1},r_k)$ cannot provide more information about $z_k$ than conditioning directly on $(z_{k-1},z_{k,\tau})$.
Therefore,
\begin{equation}
    \mathcal{L}_{z,r}
    \geq
    \mathbb{E}\!\left[
        \left\|
        z_k
        -
        \mathbb{E}[z_k\mid z_{k-1},z_{k,\tau}]
        \right\|_2^2
    \right].
\end{equation}
Under the Gaussian assumption in Eq.~\ref{eq:app_gaussian_assumption} and the corruption process in Eq.~\ref{eq:app_corrupted_future}, standard Gaussian conditioning gives
\begin{equation}
    \operatorname{Cov}
    (z_k\mid z_{k-1},z_{k,\tau})
    =
    \left(
        \Sigma^{-1}
        +
        \frac{(1-\tau)^2}{\tau^2}I
    \right)^{-1}.
\end{equation}
The minimum mean-squared prediction error is the trace of this conditional covariance.
Since $\Sigma\succ0$ and $\tau>0$, the conditional covariance is positive definite, yielding the strictly positive lower bound in Eq.~\ref{eq:app_positive_prediction_error}.
\end{proof}

Together, Propositions~\ref{prop:transition_information} and~\ref{prop:corruption_copying} characterize the complementary roles of the transition representation and future-state corruption.
Future prediction encourages the Action-DiT pathway to capture information about the transition beyond what is available from the preceding state, while future corruption prevents the objective from being solved by the trivial shortcut of directly transmitting the clean prediction target.

\section{Additional Implementation Details}
\label{app:implementation_details}

This section provides additional implementation details for native action-prior pre-training, downstream post-training, benchmark evaluation, and action-only inference.

\subsection{Native Action-Prior Pre-training}
\label{app:pretrain}

\paragraph{Pre-training data.}
We train on $235{,}790$ observation-only manipulation episodes totaling $1{,}054.3$ hours of video from Open X-Embodiment~\citep{o2024open}, AgiBotWorld~\citep{bu2025agibot}, and EgoDex~\citep{hoque2025egodex}.
For AgiBotWorld and EgoDex, we follow the OSCAR curation procedure~\citep{wu2026oscar}, filtering short or low-quality interactions based on camera motion, manipulator activity, and hand visibility, followed by near-duplicate removal using visual and manipulator-trajectory similarity.
Videos are divided into non-overlapping $65$-frame clips, rendered on a $448\times448$ canvas, and temporally subsampled with a stride of $4$, yielding approximately $1.41$M training clips.
Each clip is encoded by the frozen spatiotemporal VAE into five latent blocks, resulting in approximately $7.0$M latent blocks for pre-training.

\paragraph{Optimization.}
Native action-prior pre-training uses only the future-video flow-matching objective in Eq.~\ref{eq:pretraining_objective}.
As described in Sec.~\ref{sec:method_pretraining}, the Video-DiT, text encoder, and VAE remain frozen, while only the Action-DiT is optimized.
We train for $14{,}000$ optimization steps on 128 GPUs using bf16 mixed precision and AdamW with $\beta_1=0.9$, $\beta_2=0.95$, and weight decay $10^{-2}$.
We use a fixed learning rate of $1\times10^{-4}$ and a global batch size of $512$.

\begin{table}[!t]
\centering
\small
\caption{\textbf{Downstream training hyperparameters on LIBERO.}
We first perform action-free target-domain cold start and subsequently jointly post-train the Video-DiT and Action-DiT using action-labeled demonstrations.}
\label{tab:downstream_hparams_libero}
\begin{tabular}{lcc}
\toprule
 & \textbf{Cold Start} & \textbf{Joint Post-training} \\
\midrule
Trainable            & Video-DiT & Video-DiT + Action-DiT \\
Action horizon       & --- & 16 \\
Learning rate        & $4\times10^{-5}$ (video) & $1\times10^{-4}$ (video), $2\times10^{-4}$ (action) \\
Epochs               & 5 & 10 \\
Scheduler            & \multicolumn{2}{c}{cosine, $5\%$ warm-up} \\
Global batch size        & \multicolumn{2}{c}{256} \\
Observation
    & \multicolumn{2}{c}{$224\times448$ composite (agent-view + wrist-view)} \\
\bottomrule
\end{tabular}

\end{table}

\begin{table}[!t]
\centering
\small
\caption{\textbf{Downstream training hyperparameters on RoboTwin~2.0.}
We first perform action-free target-domain cold start and subsequently jointly post-train the Video-DiT and Action-DiT using Clean-domain demonstrations.}
\label{tab:downstream_hparams_robotwin}
\begin{tabular}{lcc}
\toprule
 & \textbf{Cold Start} & \textbf{Joint Post-training} \\
\midrule
Trainable            & Video-DiT & Video-DiT + Action-DiT \\
Action horizon       & --- & 16 \\
Learning rate        & $2\times10^{-4}$ (video) & $1\times10^{-4}$ (video), $3\times10^{-4}$ (action) \\
Epochs               & 1 & 5 \\
Scheduler            & \multicolumn{2}{c}{cosine, $5\%$ warm-up} \\
Global batch size       & \multicolumn{2}{c}{1024} \\
Observation
    & \multicolumn{2}{c}{$384\times320$ composite (head-view + two wrist-views)} \\
\bottomrule
\end{tabular}

\end{table}

\subsection{Downstream Post-training}
\label{app:downstream_adaptation}

For downstream adaptation, we first perform an action-free cold start to adapt the Video-DiT to the target visual domain, following~\cite{chen2026mv}.
During this stage, the pretrained Action-DiT remains frozen, while the Video-DiT is optimized using future-video flow matching on target-domain videos without action labels.
This stage serves solely as target-domain visual initialization and introduces no additional action supervision.
We then unfreeze both experts and jointly post-train the Video-DiT and Action-DiT using Eq.~\ref{eq:posttraining_objective} with $\lambda_v=\lambda_a=1$.

For both LIBERO and RoboTwin~2.0, downstream training is initialized from the same checkpoint pretrained with native action-prior learning.
Benchmark-specific cold-start and joint post-training configurations are summarized in Tabs.~\ref{tab:downstream_hparams_libero} and~\ref{tab:downstream_hparams_robotwin}, respectively.

\subsection{Detailed Transition-Structured Attention Masks}
\label{app:attention_mask}

Fig.~\ref{fig:detailed_attention_mask} provides a detailed visualization of the transition-structured attention masks introduced in Sec.~\ref{sec:method_pretraining}.
For a sequence of visual transitions, we construct one segment for each transition, where segment $k$ contains the clean preceding state $z_{k-1}$, the noised future state $z_{k,\tau_v}$, and the corresponding Action-DiT token group $x_k^a$.
Different segments are fully isolated from one another, ensuring that each segment models its corresponding transition without accessing information from other transitions.

Within each segment, the clean preceding-state queries attend only to themselves.
During pre-training, the noised-future and action-token queries attend to all token groups within the segment, allowing future-video supervision to propagate through the Action-DiT.
During post-training, noised-future video queries attend only to visual tokens, while action queries retain access to both visual and action tokens, yielding the asymmetric interaction described in Sec.~\ref{sec:method_posttraining}.

\begin{figure}[!t]
    \centering
    \includegraphics[width=\linewidth]{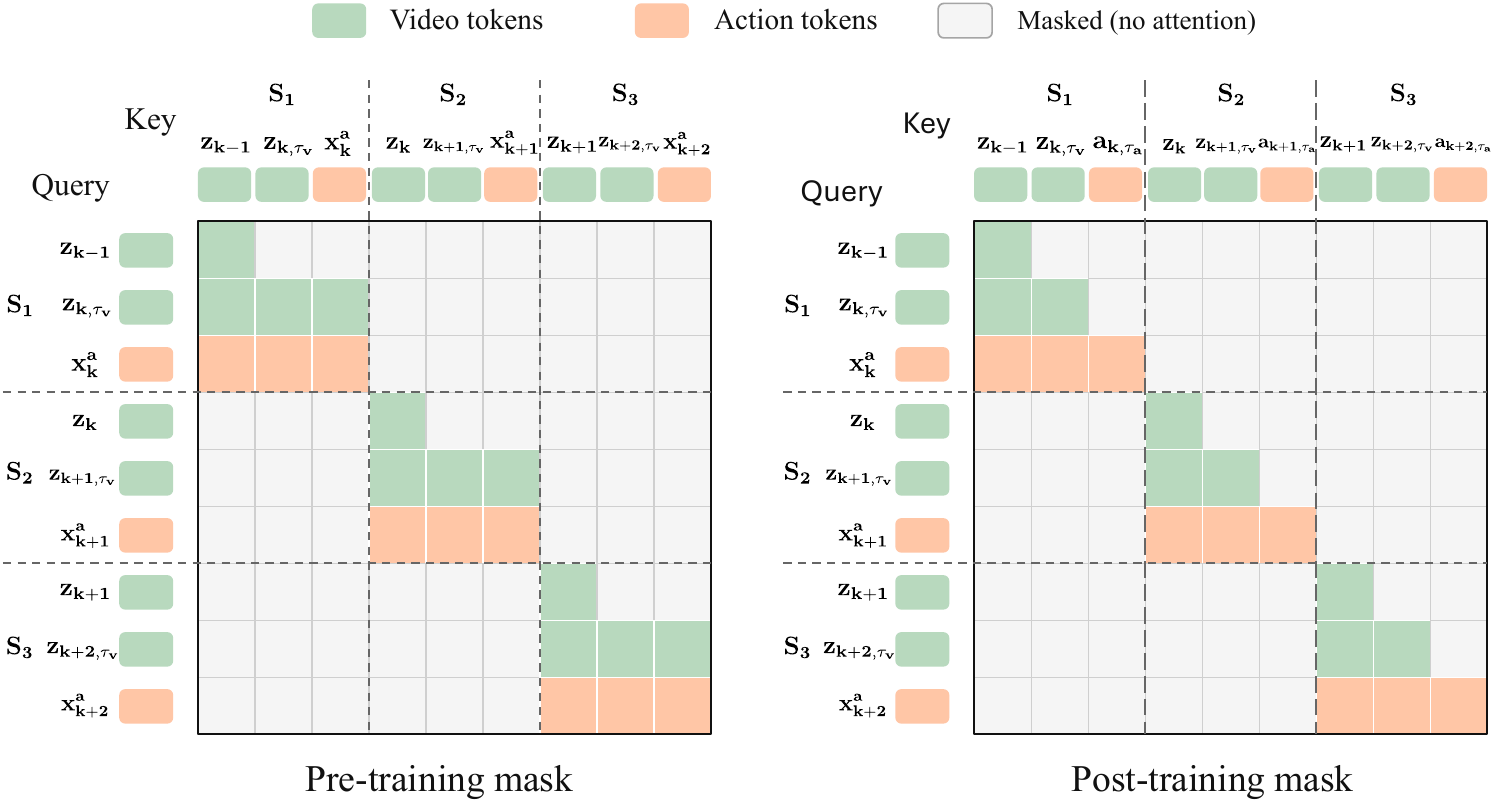}
\caption{\textbf{Transition-structured attention masks.}
The pre-training mask (\textit{left}) and post-training mask (\textit{right}) are shown for three consecutive transition segments $S_1$--$S_3$.
Each segment contains a clean preceding visual state, a noised future state, and the corresponding Action-DiT tokens, with attention isolated across segments.
Pre-training allows noised-future and action-token queries to attend to all tokens within their segment, whereas during post-training, noised-future video queries attend only to visual tokens.
}
    \label{fig:detailed_attention_mask}
\end{figure}

\subsection{Benchmark and Evaluation Details}
\label{app:evaluation_details}

\paragraph{Baselines.}
We compare \method{} with representative action policies and recent world action models reported in Tab.~\ref{tab:main_results}.
For RoboTwin~2.0, baseline results follow the Clean-to-Random evaluation reported by MV-WAM~\citep{chen2026mv}.
Unless otherwise specified, all other baseline results are taken from the corresponding papers.

\paragraph{LIBERO.}
LIBERO~\citep{liu2023libero} is built on the robosuite simulator and comprises four task suites covering different manipulation and generalization settings.
\textit{LIBERO-Spatial} varies the spatial relationships between objects, \textit{LIBERO-Object} varies the object category to be manipulated, and \textit{LIBERO-Goal} fixes the objects and their spatial arrangement while varying the task goal.
The fourth suite, \textit{LIBERO-Long}, focuses on long-horizon tasks that combine these factors.
Each suite contains ten tasks with 50 demonstrations per task.
Following the standard protocol~\citep{zhang2026imagewam,yuan2026fastwam}, we train a single multi-task policy on all 40 tasks using 50 demonstrations per task, totaling $2{,}000$ trajectories, and evaluate each task over 50 rollouts with a 7-D action space.

\paragraph{LIBERO-Plus.}
LIBERO-Plus~\citep{fei25libero-plus} reconstructs the 40 LIBERO evaluation tasks under controlled perturbations along seven factors: \emph{object layout}, \emph{camera viewpoints}, \emph{robot initial states}, \emph{language instructions}, \emph{lighting conditions}, \emph{background textures}, and \emph{sensor noise}.
The full benchmark comprises $10{,}030$ test-only task instances.
We use no LIBERO-Plus trajectories for training and directly evaluate the policy post-trained on the original LIBERO demonstrations.
Following~\cite{zhang2026imagewam}, we perform one rollout per instance and report the macro-averaged success rate across perturbation types and task suites in Sec.~\ref{app:detailed_results}.

\paragraph{RoboTwin~2.0.}
RoboTwin~2.0~\citep{chen2025robotwin} provides bimanual manipulation tasks under Clean and domain-randomized conditions.
It applies structured domain randomization along five axes:  \textit{task-irrelevant distractor objects},  \textit{background textures},  \textit{lighting conditions},  \textit{tabletop height}, and \textit{language instructions}.
We evaluate the 50-task multi-task setting and, following~\cite{chen2026mv}, train exclusively on 50 Clean-domain demonstrations per task, totaling $2{,}500$ action-labeled trajectories.
No Random-domain demonstrations are used for training, making the Random domain an OOD evaluation relative to Clean-domain training.
We evaluate 100 rollouts per task under both conditions with a 14-D action space.

\subsection{Action-Only Inference}
\label{app:action_inference}

The asymmetric transition attention introduced in Sec.~\ref{sec:method_posttraining} makes the Video-DiT independent of the evolving action sample, allowing its representations to be reused across action-denoising steps.
At each replanning step, we initialize the future-video and action streams with independent Gaussian noise and perform a single joint forward pass at $\tau_v=\tau_a=1$.
During this pass, we cache the layer-wise Video-DiT keys and values.
All subsequent denoising steps evaluate only the Action-DiT while reusing the cached visual representations.

We integrate the action flow using $S=10$ uniform Euler steps for both LIBERO and RoboTwin~2.0.
The future-video latent is neither iteratively denoised nor decoded, and the visual cache is recomputed once at each replanning step.
Algorithm~\ref{alg:action_only_inference} summarizes the inference procedure.

\begin{algorithm}[t]
\caption{Action-only inference with cached Video-DiT representations}
\label{alg:action_only_inference}
\begin{algorithmic}[1]
\REQUIRE observation latent $z_0$; instruction $l$; proprioception $q$; flow steps $S$; step size $\Delta\tau=1/S$
\ENSURE action chunk $\hat{a}$ of horizon $H$
\STATE $\epsilon^v,\epsilon^a \sim \mathcal{N}(0,I)$
\STATE $a \gets \epsilon^a$; \quad $\tau_a \gets 1$
\STATE $(u^a,\mathcal{C}_V) \gets
\operatorname{JointStep}\!\left(
[z_0,\epsilon^v],a,\tau_v=1,\tau_a=1 \mid l,q
\right)$
\STATE $a \gets a-\Delta\tau\,u^a$
\STATE $\tau_a \gets \tau_a-\Delta\tau$
\FOR{$s=2$ \TO $S$}
    \STATE $u^a \gets v^a_{\theta_A}\!\left(a,\tau_a\mid\mathcal{C}_V,l,q\right)$
    \STATE $a \gets a-\Delta\tau\,u^a$
    \STATE $\tau_a \gets \tau_a-\Delta\tau$
\ENDFOR
\STATE \textbf{return} $\hat{a}\gets a$
\end{algorithmic}
\end{algorithm}

\section{Additional Experimental Results}
\label{app:additional_results}

This section provides additional analyses of native action-prior pre-training, qualitative visualizations of learned transition priors, detailed benchmark breakdowns, and RoboTwin~2.0 policy rollouts.

\subsection{Analysis of Native Action-Prior Pre-training}
\label{app:pretraining_analysis}

\begin{figure}[!t]
    \centering

    \includegraphics[width=0.9\linewidth]{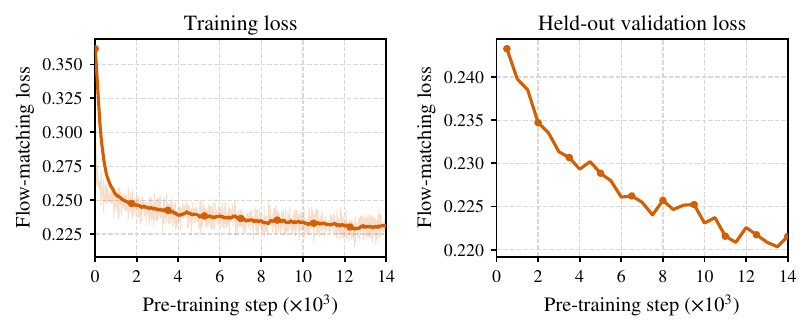}
    \captionsetup{skip=1pt}
    \caption{\textbf{Native action-prior pre-training dynamics.}
    Training and held-out future-video flow-matching losses over pre-training.
    The training curve shows the raw per-step loss (faint) and its exponential moving average (solid), while the held-out loss is evaluated periodically without smoothing.}
    \label{fig:pretrain_curves}

    \vspace{4pt}

    \includegraphics[width=\linewidth]{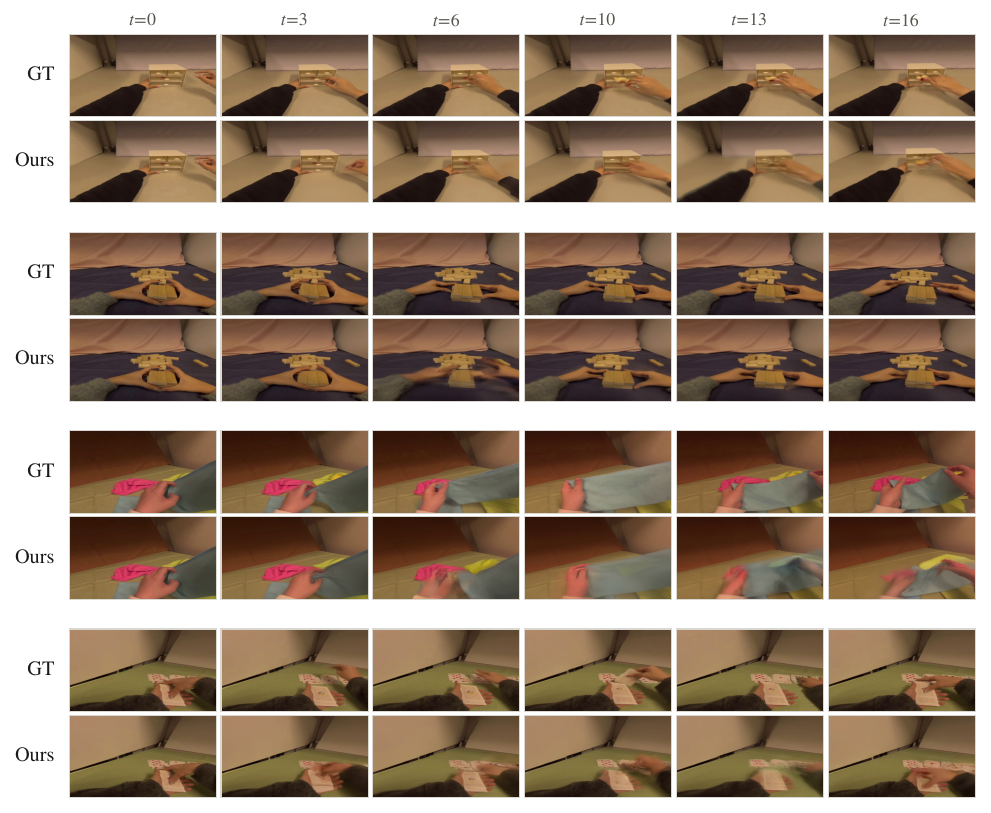}
    \caption{\textbf{Future-video prediction after native action-prior pre-training.}
    Predictions on four held-out egocentric manipulation clips.
    Each example shows the ground-truth sequence (\textit{GT}) and the corresponding prediction from \method{} (\textit{Ours}) across block stacking, two-handed assembly, cloth folding, and card manipulation.
    }
    \label{fig:pretrain_recon}

\end{figure}

\paragraph{Pre-training dynamics.}
Fig.~\ref{fig:pretrain_curves} shows the training and held-out future-video flow-matching losses throughout native action-prior pre-training.
Both losses decrease rapidly during the first few thousand optimization steps and continue to decline gradually thereafter, indicating stable optimization.

\begin{figure}[!t]
    \centering
    \includegraphics[width=\linewidth]{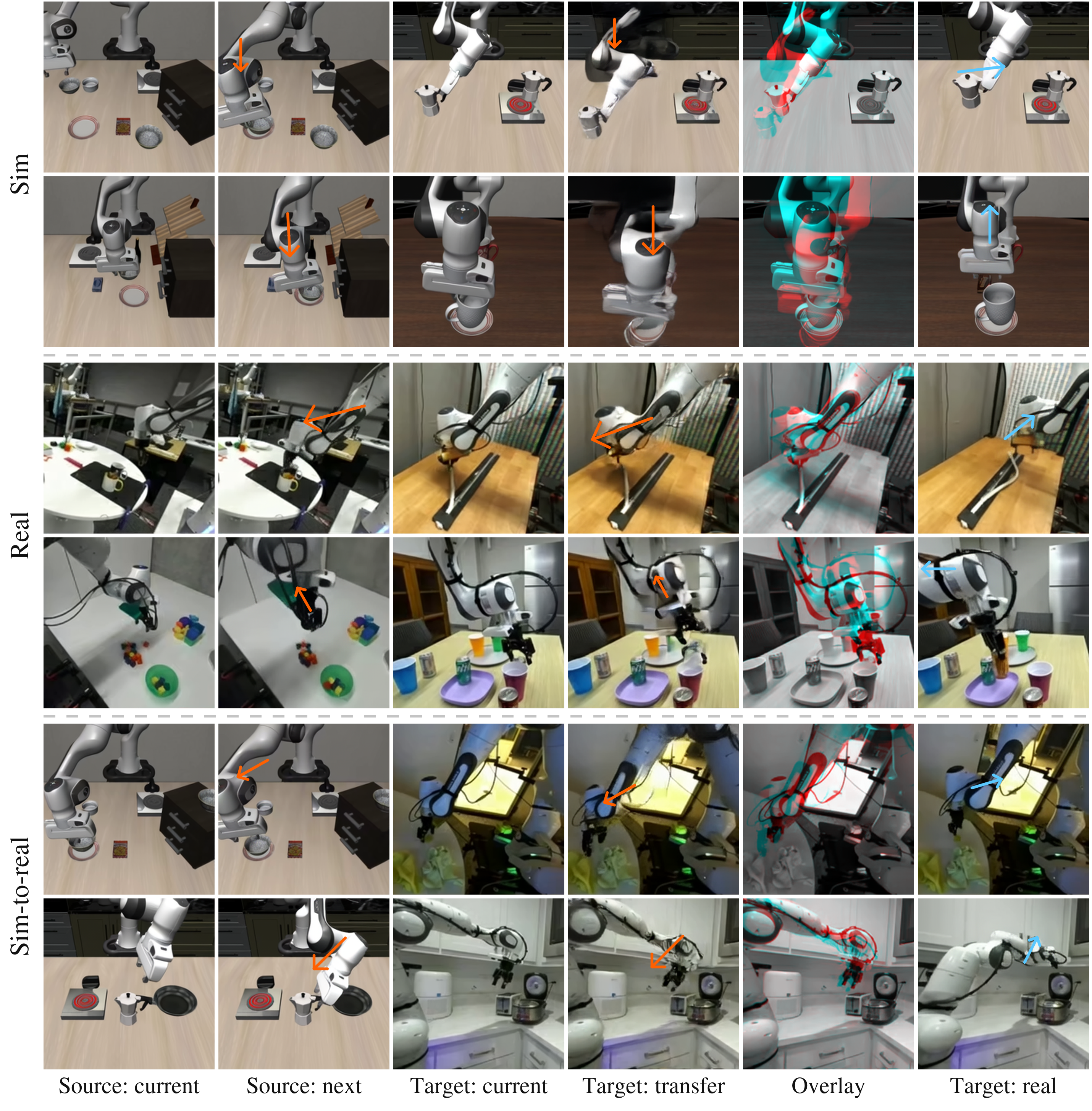}
    \caption{\textbf{Additional qualitative visualizations of motion transfer across domains.}
    Each example transfers the transition encoded by a source pair (\textit{Source: current} $\rightarrow$ \textit{Source: next}) to a different target observation.
    The source transition specifies the underlying motion cue (orange arrows), which is applied to \textit{Target: current} to obtain \textit{Target: transfer}.
    The \textit{Overlay} visualizes the transferred state relative to the current target observation, while \textit{Target: real} shows the corresponding real transition (blue arrows).
    Across simulation, real-world, and sim-to-real examples, the transferred states follow the source motion despite substantial changes in visual appearance and scene configuration.}
    \label{fig:motion_transfer_additional}
\end{figure}

\paragraph{Future-video prediction.}
Fig.~\ref{fig:pretrain_recon} visualizes future-video predictions on held-out egocentric manipulation clips after native action-prior pre-training.
Each temporally subsampled clip contains $17$ frames and is encoded into $5$ VAE latent blocks.
The first block is kept clean, while each subsequent block is independently denoised from Gaussian noise using $10$ flow-matching steps.
The predictions show that the model captures both hand motion and the evolution of manipulated objects across diverse interactions, indicating that the pretrained Action-DiT effectively supports visual-transition prediction.

\subsection{Qualitative Analysis of Learned Transition Priors}
\label{app:qualitative_transition}

Fig.~\ref{fig:motion_transfer_additional} provides additional motion-transfer examples complementing Fig.~\ref{fig:motion_transfer}.
Across source--target pairs with substantial differences in scene appearance, layout, and camera viewpoint, the transferred target states follow the motion specified by the corresponding source transitions in simulation, real-world, and sim-to-real settings.
These examples suggest that the pretrained Action-DiT captures transition-relevant action information that transfers across visual domains.

\subsection{Detailed Benchmark Results}
\label{app:detailed_results}

We report per-suite, per-perturbation, and per-task results underlying the aggregate success rates in Tab.~\ref{tab:main_results}.

\paragraph{LIBERO.}
Tab.~\ref{tab:libero_per_suite} reports success rates across the four LIBERO suites.
\method{} performs consistently across all suites and achieves $99.0\%$ success on LIBERO-Long, where several compared methods show their largest performance drop.

\begin{table}[t]
\centering
\small
\setlength{\tabcolsep}{10pt}
\caption{\textbf{Per-suite success rates (\%) on LIBERO.}
All methods are trained on 50 demonstrations per task and evaluated over 50 rollouts per task.
\textit{Avg.} denotes the macro-average across the four task suites.
Baselines with available per-suite results are reported.
The best result in each column is shown in \textbf{bold}.}
\label{tab:libero_per_suite}
\resizebox{0.8\linewidth}{!}{%
\begin{tabular}{lccccc}
\toprule
Method & Spatial & Object & Goal & Long & Avg. \\
\midrule
\rowcolor{gray!10}
\multicolumn{6}{l}{\emph{Direct action policies}} \\
$\pi_0$~\citeyearpar{black2024pi0}
    & 96.8 & 98.8 & 95.8 & 85.2 & 94.1 \\
$\pi_{0.5}$~\citeyearpar{physicalintelligence2025pi05}
    & \textbf{98.8} & 98.2 & 98.0 & 92.4 & 96.9 \\
\addlinespace[2pt]
\rowcolor{gray!10}
\multicolumn{6}{l}{\emph{World action models}} \\
\addlinespace[2pt]
Fast-WAM~\citeyearpar{yuan2026fastwam}
    & 98.2 & \textbf{100.0} & 97.0 & 95.2 & 97.6 \\
Image-WAM~\citeyearpar{zhang2026imagewam}
    & 97.2 & 99.2 & \textbf{98.8} & 98.4 & 98.4 \\
\addlinespace[2pt]
\rowcolor{row}
\textbf{\method{}}
    & \textbf{98.8} & 99.6 & 98.6 & \textbf{99.0} & \textbf{99.0} \\
\bottomrule
\end{tabular}}

\end{table}

\begin{table}[t]
\centering
\small
\caption{\textbf{Per-perturbation success rates (\%) on LIBERO-Plus.}
Each column averages over the four LIBERO task suites, and \textit{Avg.} denotes the macro-average across the seven perturbation types.
No LIBERO-Plus trajectories are used for training.
The best result in each column is shown in \textbf{bold}.}
\label{tab:libero_plus_per_dim}
\setlength{\tabcolsep}{4pt}
\resizebox{\linewidth}{!}{%
\begin{tabular}{lcccccccc}
\toprule
Method & Camera & Robot & Language & Light & Background & Noise & Layout & Avg. \\
\midrule
$\pi_0$~\citeyearpar{black2024pi0}         & 13.8 & 6.0 & 58.8 & 85.0 & 81.4 & 79.0 & 68.9 & 53.6 \\
$\pi_0$-FAST~\citeyearpar{pertsch2025fast} & 65.1 & 21.6 & 61.0 & 73.2 & 73.2 & 74.4 & 68.8 & 61.6 \\
Fast-WAM~\citeyearpar{yuan2026fastwam}     & 16.4 & 44.5 & 68.9 & 78.2 & 53.7 & 37.7 & 60.7 & 51.5 \\
Image-WAM~\citeyearpar{zhang2026imagewam}  & \textbf{80.8} & 50.3 & 91.4 & \textbf{98.1} & \textbf{85.5} & \textbf{93.8} & 80.5 & 83.1 \\
\addlinespace[2pt]
\rowcolor{row}
\textbf{\method{}} & 72.8 & \textbf{65.6} & \textbf{93.8} & 97.9 & 82.8 & 84.8 & \textbf{87.1} & \textbf{83.5} \\
\bottomrule
\end{tabular}}

\end{table}

\paragraph{LIBERO-Plus.}
Tab.~\ref{tab:libero_plus_per_dim} reports success rates across the seven LIBERO-Plus perturbation types, with baseline results taken from~\cite{zhang2026imagewam}.
\method{} performs best on robot-initialization, language, and object-layout shifts, reaching $65.6\%$, $93.8\%$, and $87.1\%$, respectively.
Image-WAM performs better on camera, lighting, background, and sensor-noise perturbations.
Overall, \method{} achieves the highest macro-average success rate, demonstrating strong out-of-domain generalization across diverse perturbations.

\paragraph{RoboTwin~2.0.}
Tab.~\ref{tab:robotwin_per_task} reports success rates across all 50 RoboTwin~2.0 tasks.
\method{} achieves the best or tied-best Random-domain performance on $44$ of the $50$ tasks and the best or tied-best Clean-domain performance on $32$ tasks.
Its aggregate Clean-to-Random performance drop is also substantially smaller than those of the compared baselines, indicating stronger robustness to the Random-domain shift.

\begin{table}[!t]
\centering
\caption{\textbf{Per-task success rates (\%) on RoboTwin~2.0.}
All methods are trained exclusively on Clean-domain demonstrations and evaluated in both the Clean and Random (\textit{Rand.}) domains.
Baseline per-task results are taken from MV-WAM~\citep{chen2026mv}.
The final row reports the average across all 50 tasks.
The best result for each task and domain is shown in \textbf{bold}.}
\label{tab:robotwin_per_task}
\setlength{\tabcolsep}{2pt}
\resizebox{\linewidth}{!}{%
\begin{tabular}{@{}l|cc|cc|cc|cc|cc|cc|cc|cc|cc@{}}
\toprule
Task & \multicolumn{2}{c|}{DP} & \multicolumn{2}{c|}{RDT} & \multicolumn{2}{c|}{$\pi_0$} & \multicolumn{2}{c|}{UP-VLA} & \multicolumn{2}{c|}{BagelVLA} & \multicolumn{2}{c|}{HALO} & \multicolumn{2}{c|}{Fast-WAM} & \multicolumn{2}{c|}{MV-WAM} & \multicolumn{2}{c}{\textbf{Ours}} \\
\cmidrule(lr){2-3}\cmidrule(lr){4-5}\cmidrule(lr){6-7}\cmidrule(lr){8-9}\cmidrule(lr){10-11}\cmidrule(lr){12-13}\cmidrule(lr){14-15}\cmidrule(lr){16-17}\cmidrule(lr){18-19}
 & Clean & Rand. & Clean & Rand. & Clean & Rand. & Clean & Rand. & Clean & Rand. & Clean & Rand. & Clean & Rand. & Clean & Rand. & Clean & Rand. \\
\midrule
Adjust Bottle & 97.0 & 0.0 & 81.0 & 75.0 & 90.0 & 56.0 & \textbf{100.0} & 17.0 & \textbf{100.0} & 14.0 & 97.0 & 9.0 & 95.0 & 0.0 & 83.0 & 65.0 & \textbf{100.0} & \textbf{97.5} \\
Beat Block Hammer & 42.0 & 0.0 & 77.0 & 37.0 & 43.0 & 21.0 & 66.0 & 16.0 & 87.0 & 16.0 & 96.0 & 11.0 & 74.0 & 2.0 & 75.0 & \textbf{53.0} & \textbf{100.0} & 42.5 \\
Blocks Ranking RGB & 0.0 & 0.0 & 3.0 & 0.0 & 19.0 & 5.0 & 38.0 & 0.0 & 84.0 & 4.0 & 94.0 & 7.0 & 0.0 & 1.0 & 99.0 & 88.0 & \textbf{100.0} & \textbf{92.5} \\
Blocks Ranking Size & 1.0 & 0.0 & 0.0 & 0.0 & 7.0 & 1.0 & 21.0 & 0.0 & 45.0 & 2.0 & 58.0 & 2.0 & 38.0 & 0.0 & 63.0 & 43.0 & \textbf{70.0} & \textbf{60.0} \\
Click Alarmclock & 61.0 & 5.0 & 61.0 & 12.0 & 63.0 & 11.0 & 69.0 & 41.0 & 85.0 & 20.0 & 83.0 & 14.0 & \textbf{100.0} & 38.0 & 90.0 & 24.0 & \textbf{100.0} & \textbf{100.0} \\
Click Bell & 54.0 & 0.0 & 80.0 & 9.0 & 44.0 & 3.0 & 54.0 & 72.0 & \textbf{100.0} & 35.0 & \textbf{100.0} & 10.0 & \textbf{100.0} & 22.0 & 95.0 & 43.0 & \textbf{100.0} & \textbf{90.0} \\
Dump Bin Bigbin & 49.0 & 0.0 & 64.0 & 32.0 & 83.0 & 24.0 & 81.0 & 35.0 & 91.0 & 51.0 & 93.0 & 28.0 & 96.0 & 3.0 & 92.0 & 61.0 & \textbf{97.5} & \textbf{77.5} \\
Grab Roller & 98.0 & 0.0 & 74.0 & 43.0 & 96.0 & 80.0 & 99.0 & 28.0 & 99.0 & 41.0 & 95.0 & 57.0 & 95.0 & 4.0 & \textbf{100.0} & 94.0 & \textbf{100.0} & \textbf{100.0} \\
Handover Block & 10.0 & 0.0 & 45.0 & 14.0 & 45.0 & 8.0 & 4.0 & 0.0 & 38.0 & 0.0 & 81.0 & 36.0 & 5.0 & 0.0 & \textbf{83.0} & 19.0 & 82.5 & \textbf{70.0} \\
Handover Mic & 53.0 & 0.0 & 90.0 & 31.0 & 98.0 & 13.0 & 45.0 & 0.0 & 75.0 & 8.0 & 96.0 & 61.0 & \textbf{100.0} & 0.0 & 99.0 & 86.0 & 95.0 & \textbf{100.0} \\
Hanging Mug & 8.0 & 0.0 & 23.0 & 16.0 & 11.0 & 3.0 & 4.0 & 0.0 & 12.0 & 1.0 & \textbf{48.0} & 5.0 & 32.0 & 0.0 & 46.0 & 16.0 & 45.0 & \textbf{27.5} \\
Lift Pot & 39.0 & 0.0 & 72.0 & 9.0 & 84.0 & 36.0 & 20.0 & 0.0 & 87.0 & 32.0 & 95.0 & 34.0 & 89.0 & 0.0 & \textbf{100.0} & 46.0 & 97.5 & \textbf{60.0} \\
Move Can Pot & 39.0 & 0.0 & 25.0 & 12.0 & 58.0 & 21.0 & 48.0 & 0.0 & 78.0 & 0.0 & 94.0 & 15.0 & 85.0 & 7.0 & 87.0 & 24.0 & \textbf{95.0} & \textbf{25.0} \\
Move Pillbottle Pad & 1.0 & 0.0 & 8.0 & 0.0 & 21.0 & 1.0 & 51.0 & 7.0 & 92.0 & 1.0 & 76.0 & 26.0 & 92.0 & 1.0 & 90.0 & 34.0 & \textbf{100.0} & \textbf{85.0} \\
Move Playingcard Away & 47.0 & 0.0 & 43.0 & 11.0 & 53.0 & 22.0 & 79.0 & 13.0 & 92.0 & 30.0 & 89.0 & 53.0 & 99.0 & 2.0 & 98.0 & 89.0 & \textbf{100.0} & \textbf{92.5} \\
Move Stapler Pad & 1.0 & 0.0 & 2.0 & 0.0 & 0.0 & 2.0 & 8.0 & 0.0 & 27.0 & 0.0 & 45.0 & 19.0 & 35.0 & 0.0 & 24.0 & 11.0 & \textbf{65.0} & \textbf{47.5} \\
Open Laptop & 49.0 & 0.0 & 59.0 & 32.0 & 85.0 & 46.0 & 86.0 & 21.0 & \textbf{96.0} & 37.0 & 89.0 & 37.0 & 90.0 & 8.0 & 94.0 & 59.0 & 92.5 & \textbf{67.5} \\
Open Microwave & 5.0 & 0.0 & 37.0 & 20.0 & 80.0 & \textbf{50.0} & 2.0 & 7.0 & 0.0 & 0.0 & \textbf{86.0} & 24.0 & 44.0 & 1.0 & 50.0 & 2.0 & 35.0 & 5.0 \\
Pick Diverse Bottles & 6.0 & 0.0 & 2.0 & 0.0 & 27.0 & 6.0 & 52.0 & 18.0 & 83.0 & 34.0 & 76.0 & 17.0 & 69.0 & 4.0 & 81.0 & 50.0 & \textbf{90.0} & \textbf{57.5} \\
Pick Dual Bottles & 24.0 & 0.0 & 42.0 & 13.0 & 57.0 & 12.0 & 82.0 & 31.0 & 93.0 & 56.0 & 82.0 & 30.0 & 76.0 & 9.0 & 93.0 & 56.0 & \textbf{100.0} & \textbf{75.0} \\
Place A2B Left & 2.0 & 0.0 & 3.0 & 1.0 & 31.0 & 1.0 & 74.0 & 4.0 & 79.0 & 12.0 & 68.0 & 8.0 & 76.0 & 1.0 & 90.0 & 69.0 & \textbf{95.0} & \textbf{77.5} \\
Place A2B Right & 13.0 & 0.0 & 1.0 & 1.0 & 27.0 & 6.0 & 56.0 & 1.0 & 81.0 & 11.0 & 52.0 & 9.0 & 84.0 & 3.0 & 87.0 & 68.0 & \textbf{92.5} & \textbf{90.0} \\
Place Bread Basket & 14.0 & 0.0 & 10.0 & 2.0 & 17.0 & 4.0 & 63.0 & 20.0 & 90.0 & 29.0 & 90.0 & 26.0 & \textbf{96.0} & 4.0 & 92.0 & 71.0 & 92.5 & \textbf{90.0} \\
Place Bread Skillet & 11.0 & 0.0 & 5.0 & 1.0 & 23.0 & 1.0 & 71.0 & 16.0 & 91.0 & 26.0 & 85.0 & 23.0 & 85.0 & 3.0 & 90.0 & 60.0 & \textbf{92.5} & \textbf{87.5} \\
Place Burger Fries & 72.0 & 0.0 & 50.0 & 27.0 & 80.0 & 4.0 & 97.0 & 26.0 & 99.0 & 11.0 & 99.0 & 37.0 & 96.0 & 8.0 & 90.0 & 79.0 & \textbf{100.0} & \textbf{100.0} \\
Place Can Basket & 18.0 & 0.0 & 19.0 & 6.0 & 41.0 & 5.0 & 20.0 & 0.0 & 63.0 & 0.0 & 68.0 & 34.0 & 58.0 & 0.0 & \textbf{80.0} & 44.0 & \textbf{80.0} & \textbf{60.0} \\
Place Cans Plasticbox & 40.0 & 0.0 & 6.0 & 5.0 & 34.0 & 2.0 & 66.0 & 24.0 & 94.0 & 5.0 & 98.0 & 47.0 & 93.0 & 2.0 & \textbf{99.0} & 48.0 & 97.5 & \textbf{90.0} \\
Place Container Plate & 41.0 & 0.0 & 78.0 & 17.0 & 88.0 & 45.0 & 86.0 & 48.0 & \textbf{100.0} & 58.0 & 96.0 & 22.0 & 98.0 & 17.0 & 96.0 & \textbf{86.0} & 97.5 & 77.5 \\
Place Dual Shoes & 8.0 & 0.0 & 4.0 & 4.0 & 15.0 & 0.0 & 45.0 & 0.0 & 57.0 & 0.0 & 15.0 & 3.0 & 24.0 & 0.0 & 42.0 & 24.0 & \textbf{80.0} & \textbf{52.5} \\
Place Empty Cup & 37.0 & 0.0 & 56.0 & 7.0 & 37.0 & 11.0 & 74.0 & 27.0 & 97.0 & 34.0 & 95.0 & 28.0 & 98.0 & 10.0 & 92.0 & \textbf{86.0} & \textbf{100.0} & 85.0 \\
Place Fan & 3.0 & 0.0 & 12.0 & 2.0 & 20.0 & 10.0 & 31.0 & 1.0 & 62.0 & 5.0 & 62.0 & 9.0 & 75.0 & 1.0 & 76.0 & 42.0 & \textbf{85.0} & \textbf{80.0} \\
Place Mouse Pad & 0.0 & 0.0 & 1.0 & 0.0 & 7.0 & 1.0 & 27.0 & 0.0 & 46.0 & 14.0 & 51.0 & 12.0 & 71.0 & 0.0 & 73.0 & 34.0 & \textbf{85.0} & \textbf{57.5} \\
Place Object Basket & 15.0 & 0.0 & 33.0 & 17.0 & 16.0 & 2.0 & 56.0 & 1.0 & 66.0 & 3.0 & 89.0 & 25.0 & 50.0 & 2.0 & \textbf{93.0} & \textbf{78.0} & 80.0 & 65.0 \\
Place Object Scale & 1.0 & 0.0 & 1.0 & 0.0 & 10.0 & 0.0 & 36.0 & 4.0 & 71.0 & 0.0 & 55.0 & 5.0 & 80.0 & 0.0 & \textbf{91.0} & 53.0 & 82.5 & \textbf{77.5} \\
Place Object Stand & 22.0 & 0.0 & 15.0 & 5.0 & 36.0 & 11.0 & 76.0 & 24.0 & 87.0 & 21.0 & 84.0 & 33.0 & 95.0 & 12.0 & 93.0 & 72.0 & \textbf{97.5} & \textbf{75.0} \\
Place Phone Stand & 13.0 & 0.0 & 15.0 & 6.0 & 35.0 & 7.0 & 32.0 & 0.0 & 61.0 & 2.0 & \textbf{91.0} & 10.0 & 85.0 & 0.0 & 79.0 & 52.0 & 85.0 & \textbf{72.5} \\
Place Shoe & 23.0 & 0.0 & 35.0 & 7.0 & 28.0 & 6.0 & 76.0 & 12.0 & 90.0 & 29.0 & 70.0 & 18.0 & 83.0 & 7.0 & 91.0 & 79.0 & \textbf{95.0} & \textbf{97.5} \\
Press Stapler & 6.0 & 0.0 & 41.0 & 24.0 & 62.0 & 29.0 & 79.0 & 56.0 & 94.0 & 58.0 & 92.0 & 64.0 & 58.0 & 19.0 & \textbf{99.0} & 45.0 & 90.0 & \textbf{67.5} \\
Put Bottles Dustbin & 22.0 & 0.0 & 21.0 & 4.0 & 54.0 & 13.0 & 7.0 & 0.0 & 42.0 & 10.0 & 80.0 & 13.0 & 83.0 & 2.0 & \textbf{93.0} & 43.0 & 87.5 & \textbf{75.0} \\
Put Object Cabinet & 42.0 & 0.0 & 33.0 & 18.0 & \textbf{68.0} & 18.0 & 7.0 & 0.0 & 52.0 & 0.0 & 59.0 & 8.0 & 41.0 & 0.0 & 45.0 & 25.0 & 50.0 & \textbf{37.5} \\
Rotate QRcode & 13.0 & 0.0 & 50.0 & 5.0 & 68.0 & 15.0 & 56.0 & 2.0 & 81.0 & 21.0 & 69.0 & 11.0 & 76.0 & 0.0 & 79.0 & 40.0 & \textbf{90.0} & \textbf{70.0} \\
Scan Object & 9.0 & 0.0 & 4.0 & 1.0 & 18.0 & 1.0 & 47.0 & 23.0 & 77.0 & 32.0 & 73.0 & 24.0 & 77.0 & 4.0 & 81.0 & 59.0 & \textbf{85.0} & \textbf{67.5} \\
Shake Bottle Horizontally & 59.0 & 18.0 & 84.0 & 51.0 & 99.0 & 51.0 & \textbf{100.0} & 68.0 & \textbf{100.0} & 73.0 & \textbf{100.0} & 66.0 & \textbf{100.0} & 41.0 & \textbf{100.0} & 97.0 & \textbf{100.0} & \textbf{100.0} \\
Shake Bottle & 65.0 & 8.0 & 74.0 & 45.0 & 97.0 & 60.0 & 98.0 & 54.0 & \textbf{100.0} & 74.0 & 98.0 & 73.0 & \textbf{100.0} & 49.0 & \textbf{100.0} & 98.0 & \textbf{100.0} & \textbf{100.0} \\
Stack Blocks Three & 0.0 & 0.0 & 2.0 & 0.0 & 17.0 & 0.0 & 8.0 & 0.0 & 45.0 & 5.0 & 96.0 & 37.0 & 0.0 & 0.0 & 97.0 & 51.0 & \textbf{97.5} & \textbf{87.5} \\
Stack Blocks Two & 7.0 & 0.0 & 21.0 & 2.0 & 42.0 & 1.0 & 61.0 & 0.0 & 95.0 & 6.0 & \textbf{100.0} & 60.0 & 2.0 & 0.0 & \textbf{100.0} & 75.0 & \textbf{100.0} & \textbf{95.0} \\
Stack Bowls Three & 63.0 & 0.0 & 51.0 & 17.0 & 66.0 & 24.0 & 42.0 & 1.0 & 63.0 & 13.0 & \textbf{92.0} & 25.0 & 77.0 & 1.0 & 86.0 & \textbf{68.0} & 80.0 & 62.5 \\
Stack Bowls Two & 61.0 & 0.0 & 76.0 & 30.0 & 91.0 & 41.0 & 69.0 & 12.0 & 90.0 & 52.0 & \textbf{98.0} & 49.0 & 96.0 & 11.0 & 90.0 & 81.0 & 97.5 & \textbf{90.0} \\
Stamp Seal & 2.0 & 0.0 & 1.0 & 0.0 & 3.0 & 4.0 & 34.0 & 2.0 & 77.0 & 8.0 & 60.0 & 21.0 & 68.0 & 0.0 & 73.0 & 30.0 & \textbf{90.0} & \textbf{52.5} \\
Turn Switch & 36.0 & 1.0 & 35.0 & 15.0 & 27.0 & 23.0 & 43.0 & 26.0 & 49.0 & 30.0 & \textbf{65.0} & 27.0 & 56.0 & 17.0 & 62.0 & 65.0 & 55.0 & \textbf{67.5} \\
\midrule
Average & 28.0 & 0.6 & 34.5 & 13.7 & 46.4 & 16.3 & 52.9 & 15.2 & 75.3 & 20.5 & 80.5 & 26.4 & 71.9 & 6.3 & 84.0 & 55.7 & \textbf{88.5} & \textbf{73.6} \\
\bottomrule
\end{tabular}}

\end{table}

\begin{figure}[p]
    \centering
    \includegraphics[width=0.91\linewidth]{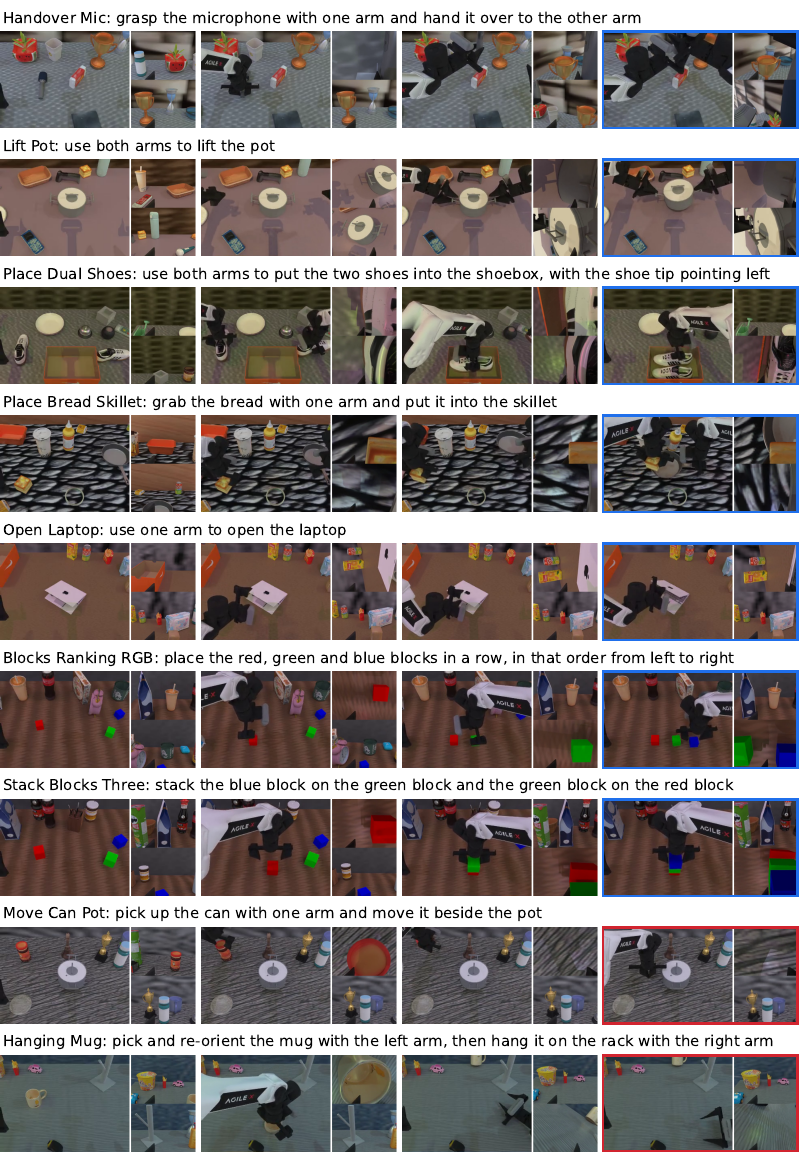}
    \vspace{-5pt}
    \caption{\textbf{Qualitative \method{} rollouts on RoboTwin~2.0.}
    We show seven successful rollouts (top) spanning diverse manipulation skills and two representative failures (bottom).
    Each row contains four temporal keyframes, with the head-camera view on the left and the two wrist-camera views stacked on the right.
    The row header gives the RoboTwin task description, and the final keyframe is outlined in blue for success and red for failure.
    All examples are from the Random (OOD) domain.}
    \label{fig:robotwin_qual}
\end{figure}

\subsection{Qualitative RoboTwin~2.0 Rollouts}
\label{app:robotwin_qualitative}

Fig.~\ref{fig:robotwin_qual} presents representative \method{} rollouts from the RoboTwin~2.0 Random domain.
All examples are drawn from the same evaluation setting as Tab.~\ref{tab:robotwin_per_task}, where the policy is trained exclusively on Clean-domain demonstrations and evaluated under domain randomization unseen during training.
Each Random-domain rollout therefore constitutes an out-of-domain evaluation under variations in distractor objects, background textures, lighting, tabletop height, and language instructions.
Videos of these rollouts are provided in our \projectpage.

\paragraph{Successful rollouts.}
The successful rollouts span diverse manipulation skills in RoboTwin~2.0.
\emph{Handover Mic} and \emph{Lift Pot} require bimanual coordination, with the former transferring an object between the two arms and the latter synchronously lifting an object with both grippers.
\emph{Place Dual Shoes} combines bimanual manipulation with a terminal pose constraint, requiring both shoes to be placed inside the shoebox with the prescribed orientation.
\emph{Place Bread Skillet} requires precise single-arm placement into a small container, while \emph{Open Laptop} requires manipulating an articulated object along its hinge.
\emph{Blocks Ranking RGB} and \emph{Stack Blocks Three} are long-horizon multi-object tasks that require grounding color references and executing multiple consecutive pick-and-place operations without disturbing previously placed blocks.
Across these tasks, \method{} completes the instructed behaviors despite substantial variation in distractors, appearance, and scene configuration.

\paragraph{Failure cases.}
The two representative failures occur during fine-grained grasping and contact-rich manipulation.
In \emph{Move Can Pot}, the reaching motion knocks the can over instead of establishing a stable grasp; the arm then executes the transport motion with an empty gripper without re-attempting the grasp, and the episode eventually reaches its step limit.
In \emph{Hanging Mug}, the policy completes the initial pick sub-goal, lifting the mug with the left arm and placing it beside the rack, but the right arm fails to continue the task by lifting the mug onto the rack, and the episode eventually reaches its step limit.

Together, these cases suggest that failures can arise during fine-grained, contact-rich interactions even when the target object and overall task progression are correctly identified.
In both examples, the policy fails to recover from a local execution error, causing it to propagate into task failure.

\subsection{Real-Robot Evaluation Details}
\label{app:real_robot}

To evaluate real-world transfer, we further post-train \method{} on DROID~\citep{khazatsky2024droid} and directly deploy the resulting policy on a physical Franka FR3 robot following~\cite{ye2026dreamzero}.
DROID is a large-scale in-the-wild manipulation dataset containing approximately $76$K teleoperated trajectories across diverse tasks and environments.
\textit{No additional fine-tuning is performed on the evaluation tasks or physical test environment.}

\begin{table}[t]
\centering
\caption{\textbf{Real-robot success rates and inference cost.}
We report successful trials out of five runs per task on a Franka FR3, with \textit{Avg.} denoting the success rate across all $15$ trials.
T1: \emph{move the cube to the left side of the bowl}; T2: \emph{put the banana in the box}; T3: \emph{put the blue cube on the red cube}.
\textit{ms\,/\,call} denotes the server-side wall-clock latency per policy call on a single H100.
The best result in each column is shown in \textbf{bold}.}
\label{tab:real_robot}
\footnotesize
\setlength{\tabcolsep}{4pt}
\renewcommand{\arraystretch}{1.08}

\begin{tabularx}{\linewidth}{
    @{}
    l
    c
    >{\centering\arraybackslash}X
    >{\centering\arraybackslash}X
    >{\centering\arraybackslash}X
    c
    @{}
}
\toprule
 & & \multicolumn{3}{c}{Successful trials $\uparrow$} & \\
\cmidrule(lr){3-5}
Model & ms\,/\,call $\downarrow$ & T1 & T2 & T3 & Avg (\%) $\uparrow$ \\
\midrule
$\pi_{0.5}$~\citep{physicalintelligence2025pi05}
& \textbf{195.7} & 0/5 & \textbf{5/5} & 3/5 & 53.3 \\

DreamZero~\citep{ye2026dreamzero}
& 3523.9 & 0/5 & \textbf{5/5} & \textbf{5/5} & 66.7 \\

\addlinespace[2pt]
\rowcolor{row}
\textbf{\method{}}
& 379.2 & \textbf{4/5} & \textbf{5/5} & \textbf{5/5} & \textbf{93.3} \\
\bottomrule
\end{tabularx}

\end{table}

\paragraph{DROID post-training.}
Starting from the native action-prior pretrained checkpoint, we post-train \method{} on DROID following the training setup of DreamZero~\citep{ye2026dreamzero}.
We use a global batch size of $256$, a cosine learning-rate schedule with $5\%$ warm-up, and a learning rate of $5\times10^{-5}$ for both experts.
The policy predicts action chunks with a horizon of $32$, and visual inputs are processed at a resolution of $352\times640$.

\paragraph{Deployment setup.}
The policy receives one wrist-camera view and two third-person RGB views, composited into the same $352\times640$ visual input used during DROID post-training.
\method{} predicts $32$-step absolute joint-position action chunks, of which the first $24$ steps are executed open-loop before replanning.
The robot controller runs at $15$\,Hz, and each trial is capped at $1{,}000$ control steps.

\paragraph{Tasks and baselines.}
We evaluate three tabletop manipulation tasks with five trials per task: T1, \emph{move the cube to the left side of the bowl}; T2, \emph{put the banana in the box}; and T3, \emph{put the blue cube on the red cube}.
We compare against $\pi_{0.5}$~\citep{physicalintelligence2025pi05} and the released DreamZero-DROID checkpoint~\citep{ye2026dreamzero}.
All policies use the same observation pipeline, communication protocol, and $15$\,Hz joint-position controller.
The scene is manually reset to a random nominal layout before each trial.
Each policy retains its own action parameterization and deployment horizon.
$\pi_{0.5}$ predicts $24$ relative-action steps and executes the first $8$ before replanning, DreamZero predicts and executes $24$ absolute-action steps, and \method{} predicts $32$ absolute-action steps and executes the first $24$.

\paragraph{Detailed task analysis.}
As summarized in Tab.~\ref{tab:real_robot}, \method{} succeeds in $14$ of $15$ trials ($93.3\%$), compared with $10/15$ ($66.7\%$) for DreamZero and $8/15$ ($53.3\%$) for $\pi_{0.5}$.
All three policies solve T2 in all five trials, showing that each policy can complete the basic manipulation under our deployment setup.

The largest difference occurs on T1, which requires grounding the instructed spatial relation: \method{} succeeds in $4/5$ trials, whereas both baselines fail in all five trials.
Inspection of the baseline rollouts suggests that these failures are not primarily caused by object localization or grasping.
In all ten baseline trials, the cube is successfully localized and grasped, but neither baseline ultimately places it to the left of the bowl.
DreamZero instead moves the cube toward the bowl and terminates with it either inside the bowl or still held above it, while $\pi_{0.5}$ places the cube to the right of the bowl or inside it.
On T3, both \method{} and DreamZero succeed in all five stacking trials, while $\pi_{0.5}$ succeeds in three.
In its two failed trials, $\pi_{0.5}$ places the blue cube beside the red cube rather than on top of it.
These results demonstrate that \method{} transfers effectively to physical robot deployment \textit{without task-specific fine-tuning}.

\paragraph{Inference cost.}
We additionally report the server-side latency of one policy call on a single H100.
\method{} requires $379.2$\,ms per call, compared with $195.7$\,ms for $\pi_{0.5}$ and $3523.9$\,ms for DreamZero.
Because the policies execute different numbers of actions before replanning, per-call latency alone does not fully reflect the effective deployment cost.
Normalizing by the number of executed actions per replanning step yields $15.8$\,ms per executed action for \method{}, $24.5$\,ms for $\pi_{0.5}$, and $146.8$\,ms for DreamZero.

\section{Limitations}
\label{app:limitations}

First, \method{} can still struggle with challenging tasks that require precise grasping, complex interactions, and other fine-grained contact-rich manipulation.
Fine-grained control and complex-instruction understanding therefore remain limitations that warrant further investigation.
Second, the policy currently lacks an explicit mechanism for recovering from local execution errors.
Developing mechanisms that enable the policy to identify such failures and recover during execution is an important direction for future work.

\end{document}